\PassOptionsToPackage{pagebackref=true}{hyperref}
\documentclass[final]{alt2026} 

\input{00_packages}
\newtheorem{reduction}{Reduction}
\newtheorem{assumption}[theorem]{Assumption}

\crefname{algorithm}{Scheme}{Schemes}
\crefname{assumption}{Assumption}{Assumptions}
\crefname{setup}{Setup}{Setups}
\crefname{reduction}{Reduction}{Reductions}
\crefname{equation}{Eq.}{Eqs.}
\crefname{appendix}{App.}{App.}

\newtheorem{setup}{Setup}

\newaliascnt{lemma}{theorem}
\newtheorem{lemma}[lemma]{Lemma}
\aliascntresetthe{lemma}

\newaliascnt{corollary}{theorem}
\newtheorem{corollary}[corollary]{Corollary}
\aliascntresetthe{corollary}

\newaliascnt{definition}{theorem}
\newtheorem{definition}[definition]{Definition}
\aliascntresetthe{definition}

\newaliascnt{proposition}{theorem}
\newtheorem{proposition}[proposition]{Proposition}
\aliascntresetthe{proposition}

\newaliascnt{claim}{theorem}
\newtheorem{claim}[theorem]{Claim}
\aliascntresetthe{claim}

\crefalias{claim}{theorem}

\newtheorem{assumption}{Assumption}

\makeatletter
\renewenvironment{proof}[1][\relax]{\par
  \normalfont \topsep6\p@\@plus6\p@\relax
  \trivlist
  \item[\hskip\labelsep\bfseries
    \ifx#1\relax \proofname\else\proofname{} #1\fi\@addpunct{.}]\ignorespaces
}
  {\jmlrQED}
\makeatother

\newcommand{\ie}{\textit{i.e., }}
\newcommand{\eg}{\textit{e.g., }}

\newcommand\nth{\textsuperscript{th} }

\newcommand{\vect}[1]{\mathbf{#1}}

\DeclareMathOperator*{\argmin}{argmin}

\newcommand{\vw}{\vect{w}}           
\newcommand{\vz}{\vect{z}}

\newcommand{\vg}{\vect{g}}
\newcommand{\vb}{\vect{b}}

\newcommand{\0}{\mat{0}}
\newcommand{\teacher}{\vw_{\star}}
\newcommand{\teachertop}{\teacher^{\top}}
\newcommand{\mat}[1]{\mathbf{#1}}
\newcommand{\X}{\mat{X}}
\newcommand{\Y}{\mat{Y}}

\newcommand{\U}{\mat{U}}

\newcommand{\V}{\mat{V}}
\newcommand{\I}{\mat{I}}
\newcommand{\A}{\mat{A}}
\newcommand{\B}{\mat{B}}
\newcommand{\C}{\mat{C}}
\newcommand{\x}{\vect{x}}

\newcommand{\vu}{\vect{u}}
\newcommand{\y}{\vect{y}}
\newcommand{\w}{\vw}

\newcommand{\rank}{\operatorname{rank}}
\newcommand{\tr}{\operatorname{tr}}

\def\mSigma{{\mat{\Sigma}}}

\def\mP{{\mat{P}}}

\newcommand\smalldots{.\hskip.8pt\!.\hskip.8pt\!.}

\newcommand{\rankavg}{\bar{r}}
\newcommand{\Loss}{\mathcal{L}}

\def\reals{\mathbb{R}}

\newcommand{\floor}[1]{\left\lfloor{#1} \right\rfloor }

\newcommand{\algmargin}{\hspace{-.9em}}
\floatname{algorithm}{Scheme}

\newcommand{\hfrac}[2]{{#1}/{#2}}
\newcommand{\norm}[1]{\left\Vert{#1}\right\Vert}

\newcommand{\bigO}[0]{\mathcal{O}}
\newcommand{\cnt}[1]{\left[{#1}\right]}
\newcommand{\explain}[1]{\left[\substack{#1}\right]}
\newcommand{\expectation}{\mathop{\mathbb{E}}}

\newcommand{\prn}[1]{\left({#1}\right)}
\newcommand{\bigprn}[1]{\big({#1}\big)}
\newcommand{\Bigprn}[1]{\Big({#1}\Big)}
\newcommand{\biggprn}[1]{\bigg({#1}\bigg)}
\newcommand{\sqprn}[1]{\left[{#1}\right]}
\newcommand{\tprn}[1]{({#1})}

\newcommand{\smallnorm}[1]{\Vert{#1}\Vert}
\newcommand{\tnorm}[1]{\smallnorm{#1}}
\newcommand{\bignorm}[1]{\big\Vert{#1}\big\Vert}
\newcommand{\Bignorm}[1]{\Big\Vert{#1}\Big\Vert}

\newcommand{\teachers}{\mathcal{W}_{\star}}
\newcommand{\tsum}{{\sum}}

\newcommand{\dataset}{S}
\newcommand{\loss}{\mathcal{L}}

\newcommand{\convexset}{\mathcal{C}}
\newcommand{\convexintersection}{\mathcal{C}_{\star}}
\newcommand{\proj}{\bm{\Pi}}

\renewcommand\dim{d}

\newenvironment{proof-sketch}{\noindent{\bf Proof sketch.}}{}

\def\secref#1{Section~\ref{#1}}

\def\eqref#1{Eq.~(\ref{#1})}

\def\appref#1{App.~\ref{#1}}

\makeatletter
\newenvironment{recall}[1][\proofname]{\par
\normalfont \topsep6\p@\@plus6\p@\relax
\trivlist
\item\relax
{\bfseries
Recall #1}%
{\bfseries\@addpunct{.}}\hspace\labelsep\ignorespaces
}
\makeatother

\usepackage{00_tau_commands}

\title[From Continual Learning to SGD and Back:
Better Rates for Continual Linear Models]{From Continual Learning to SGD and Back:\\
Better Rates for Continual Linear Models}
\usepackage{times}

\altauthor{%
 \Name{Itay Evron}\thanks{Equal contribution.
 Equal-contributing authors are listed in alphabetical order, followed by senior authors.}
 \Email{itay@evron.me}\\
 \addr Meta
 \AND
 \Name{Ran Levinstein}\footnotemark[1] \Email{ranlevinstein@gmail.com}\\
 \addr Department of Computer Science, Technion
 \AND
 \Name{Matan Schliserman}\footnotemark[1] \Email{schliserman@mail.tau.ac.il}
 \\
 \Name{Uri Sherman}\footnotemark[1] 
 \Email{urisherman@mail.tau.ac.il}
 \\
 \addr Blavatnik School of Computer Science and AI, Tel Aviv University
 \AND
 \Name{Tomer Koren}
 \Email{tkoren@tauex.tau.ac.il}\\
 \addr Blavatnik School of Computer Science and AI, Tel Aviv University, and Google Research
 \AND
 \Name{Daniel Soudry} \Email{daniel.soudry@gmail.com}\\
 \addr Department of Electrical and Computing Engineering, Technion
 \AND
 \Name{Nathan Srebro} \Email{nati@ttic.edu}\\
 \addr 
 Toyota Technological Institute at Chicago
}

\begin{document}

\maketitle

\begin{abstract}%
We study the common continual learning setup where an overparameterized model is sequentially fitted to a set of jointly realizable tasks.
We analyze forgetting, defined as the loss on previously seen tasks, after $k$ iterations. 
For continual linear models, we prove that fitting a task is equivalent to a \emph{single} stochastic gradient descent (SGD) step on a modified objective.   
We develop novel last-iterate SGD upper bounds in the realizable least squares setup and leverage them to derive new results for continual learning.
Focusing on random orderings over \(T\) tasks, we establish \emph{universal} forgetting rates, whereas existing rates depend on problem dimensionality or complexity and become prohibitive in highly overparameterized regimes. 
In continual regression with replacement, we improve the best existing rate from $\mathcal{O}((d-\bar{r})/k)$ to
$\mathcal{O}(\min(1/\sqrt[4]{k}, \sqrt {d-\bar{r}}/k, \sqrt {T\bar{r}}/k))$,\linebreak 
where $d$ is the dimensionality and $\bar{r}$ the average task rank.
Furthermore, we establish the first rate for random task orderings \emph{without} replacement. 
The resulting rate $\mathcal{O}(\min(1/\sqrt[4]{T},\, (d-\bar{r})/T))$ shows that randomization alone, without task repetition, prevents catastrophic forgetting in sufficiently long task sequences.
Finally, we prove a matching $\mathcal{O}(1/\sqrt[4]{k})$ forgetting rate for continual linear \emph{classification} on separable data. 
Our universal rates extend to broader methods, such as block \linebreak Kaczmarz and POCS, illuminating their loss convergence under i.i.d.~and single-pass orderings.
\end{abstract}


\section{Introduction}

In continual learning (CL), tasks are presented sequentially, one at a time.
The goal is for the learner to adapt to the current task---\eg by fine-tuning using gradient-based algorithms---while retaining knowledge from previous tasks.
A central challenge in this setting is termed
\emph{catastrophic forgetting}, where expertise from earlier tasks is lost when adapting to newer ones.
Forgetting is influenced by factors such as task similarity and overparameterization \citep{goldfarb2024theJointEffect},
and is also related to trade-offs like the plasticity-stability dilemma \citep{mermillod2013stability}. 
CL is becoming increasingly important with the rise of foundation models, where retraining is prohibitively expensive and data from prior tasks is often unavailable, \eg due to privacy or data retention constraints.


Previous work has shown, both analytically \citep[e.g.,][]{evron2022catastrophic,evron23continualClassification,swartworth2023nearly,jung2025convergence,cai2025lastIterate} and empirically \citep{lesort2022scaling,hemati2024continual}, that forgetting diminishes when tasks are ordered randomly or cyclically.
Task orderings can be explored from multiple perspectives:
as a strategy to mitigate forgetting (\eg by actively ordering an agent's learning environments); as a naturally occurring phenomenon, such as periodic trends in\linebreak e-commerce; or as a means to model popular CL benchmarks, such as randomly split datasets.


Our work focuses on a widely studied analytical setting---realizable continual linear regression,\footnote{\label{fn:setting}%
While simple, continual linear regression captures key factors in CL,
\eg task similarity \citep{hiratani2024disentangling,tsipory2025greedy}, task recurrence \citep{evron2022catastrophic}, overparameterization \citep{goldfarb2023analysis},
and algorithmic effects \citep{doan2021NTKoverlap,peng2023ideal}.
We follow prior work 
analyzing continual \emph{optimization} dynamics under the assumption that training data across tasks are jointly realizable
\citep{evron2022catastrophic,evron23continualClassification}.
In contrast, \emph{statistical} formulations allow label noise but assume i.i.d.~features \citep{lin2023theory,banayeeanzade2025theoretical} or commutative covariances \citep{li2023fixed,zhao2024statistical}---while our analysis applies to arbitrary data matrices.
} where $T$ tasks are learned sequentially over $k$ iterations in a uniform \emph{random ordering}.
\citet{evron2022catastrophic} established that the worst-case expected forgetting lies between $\Omega\prn{\hfrac{1}{k}}$ and $\bigO\prn{\hfrac{(d-\rankavg)}{k}}$, where $d$ is the problem dimensionality, and $\rankavg$ the average rank of individual data matrices.
This raises a fundamental question, critical in highly overparameterized regimes:
\emph{Does worst-case forgetting necessarily scale with dimensionality, and if so, is the dependence indeed linear?}


To this end, we bridge continual learning and last-iterate stochastic gradient descent (SGD) analysis.
We revisit an established connection between continual linear regression and the Kaczmarz method for solving systems of linear equations \citep{karczmarz1937angenaherte, evron2022catastrophic}. 
Given rank-$1$ tasks, each update of these methods is known to reduce to a {normalized} stochastic gradient step, fully minimizing the current task's least squares objective using a ``stepwise-optimal'' step size.
Extending to general rank, we prove that learning an \emph{entire} task in continual linear regression is equivalent to a \emph{single} SGD step on a modified objective with a fixed, stepwise-optimal step size.

Motivated by this, we prove convergence rates for the last iterate of fixed-step-size SGD that, crucially, hold for a broad range of step sizes not covered by prior work~\citep[e.g.,][]{ge2019stepDecay,berthier2020tight,zou2021benign,wu2022last}.
Specifically, prior results either hold only for the average iterate \citep[e.g.,][]{bach2013non} or small step sizes bounded away from the stepwise-optimal step size crucial for our continual setup \citep[e.g.,][]{varre2021last}.
We overcome this challenge by carefully refining analysis techniques for SGD \citep{srebro2010smoothness,shamir13sgd} to accommodate a wider range of step sizes, including the stepwise-optimal one.

Applying our last-iterate analysis to continual regression, we tighten the existing forgetting rate and establish the first dimension-independent rate
(see \cref{tab:comparison}).
Furthermore, we provide the first rate for random task orderings \emph{without} replacement, 
proving that task repetition is not obligatory to guarantee convergence when ${k=T\to \infty}$, thus highlighting the effect of randomization as compared to repetition.
Our results also yield novel rates for the related Kaczmarz and NLMS methods.

\vspace{-0.25em}
\begin{figure}[h!]
    \centering
    \includegraphics[width=0.99\linewidth]{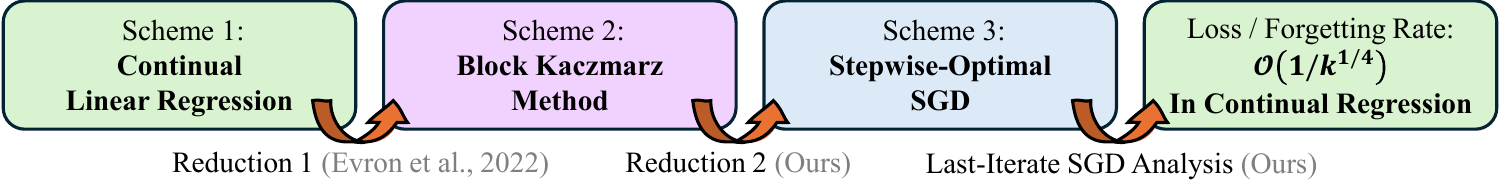}
    \vspace{-0.1em}
    \captionsetup{font=small}
    \caption{Analysis Flow Leading to Our Improved Regression Rates---From CL to SGD and Back.}
    \label{fig:regression_reductions}
    \vspace{-0.25em}
\end{figure}

Finally, by proving a matching 
rate for the squared loss of the broader Projection Onto Convex Sets framework \citep{gubin1967methodProjections}, we extend our results to continual linear \emph{classification} on separable data, 
providing this setting's first universal rate, independent of the problem's ``complexity''.


\paragraph{Summary of Contributions.}

To summarize, our main contributions in this paper are:
\begin{itemize}[leftmargin=0.45cm, itemindent=0cm, itemsep=0pt,labelsep=0.2cm,topsep=4pt]
    
    \item
    We establish new reductions from continual linear models to SGD with a rather large, \linebreak
    “stepwise-optimal” step size, generalizing results from prior work—limited to rank-$1$ tasks—to arbitrary rank. This enables last-iterate analysis for studying forgetting.

    
    \item 
    We provide novel last-iterate SGD analysis for a realizable least squares setup, yielding the first informative rates for fixed step sizes large enough to support the reductions to continual learning.
    
    \item 
   Our main results are improved loss and forgetting rates in both continual linear regression and classification 
   (see \cref{tab:comparison,tab:classification}, respectively), which
    (i) are {dimensionality-independent} and hold even in highly overparameterized regimes,
    previously uncovered by existing rates; \linebreak
    and
    (ii)
    extend to {without-replacement} orderings, revealing that task repetition is not required to mitigate forgetting.

\end{itemize}


\section{Main Setting: Continual Linear Regression}
\label{sec:setting}

We focus primarily on the fundamental continual linear regression setting, widely-studied in theoretical work.
This setting is easy-to-analyze, yet often sheds light on important CL phenomena.\textsuperscript{\ref{fn:setting}}

\paragraph{Notation.}
Boldface denotes vectors and matrices.
We use $\norm{\cdot}$ for Euclidean, spectral, or operator norms. 
%
$\X^{+}$ denotes the Moore--Penrose inverse.
Finally, we define $\cnt{n}\triangleq \{1,\dots,n\}$.

\vspace{4pt}

Formally, we are given a collection of $T$ linear regression tasks, $(\X_1, \y_1),\dots, (\X_T, \y_T)$,
where
$\X_{m}\!\in\!\reals^{n_m\!\times d},
\y_{m}\!\in\!\reals^{n_m}$.
Over $k$ iterations, tasks are learned under a \emph{task ordering}
$\tau: \cnt{k}\!\to\!\cnt{T}$,
and we focus on random orderings studied in, \eg  
\citet{evron2022catastrophic,evron23continualClassification,jung2025convergence}.

\begin{definition}[Random Task Ordering]\label{def:random_ordering}
A random ordering selects tasks uniformly at random from the task collection $\sqprn{T}$, 
\ie
$\tau(1),
    \dots,\tau(k) \sim \Unif\prn{\sqprn{T}}$,
with or without replacement.
\end{definition}

We study a direct learning scheme which minimizes the sum of squared errors for the current regression task,\footnote{
This objective is natural for regression;
our analysis also extends to the \emph{mean} squared error 
(refining our $R$).
}
without mitigating forgetting algorithmically (\eg with replay).
This scheme \linebreak (i) illuminates ``raw'' continual dynamics of gradient-based algorithms,
and (ii) roughly captures linear dynamics of deep networks in the neural tangent kernel regime \citep{jacot2018ntk}.

\vspace{-0.1em}

{\centering
\begin{algorithm}[H]
   \caption{\strut Continual Linear Regression (to Convergence)
   \strut
    \label{proc:regression_to_convergence}}
\begin{algorithmic}
\vspace{-0.1em}
   \STATE {
   \algmargin Initialize} 
   $\w_0 = \0_{\dim}$
   \STATE {
   \algmargin 
   For each iteration $t=1,\dots,k$:
   }
   \STATE {
   \algmargin \hspace{1em}
   $\w_t
   \leftarrow$
   Start from $\w_{t-1}$ and
   minimize the current task's loss
   $
   \loss_{\tau(t)} (\w)
   \triangleq
   {\frac{1}{2}}
   \norm{\X_{\tau({t})} \w - \y_{\tau({t})}}^2$ 
   } 
   \\
   \hspace{2.75em}
   with (S)GD to convergence\footnotemark
   \STATE {
   \algmargin 
   Output $\w_k$
   } 
\vspace{-0.1em}
\end{algorithmic}
\end{algorithm}
%
}
\vspace{-0.6em}

\footnotetext{%
   Learning to convergence facilitates the analysis, but other analytical choices exist \citep[see][]{jung2025convergence}.
   }

\noindent This scheme was previously linked to the Kaczmarz method and, in a special case, to normalized SGD \citep{evron2022catastrophic}. 
In \cref{sec:reductions}, we develop these connections to enable novel analysis.

\pagebreak

Our main assumption is the existence of \emph{offline solutions} that 
fit the training data of all $T$ tasks jointly,
as assumed in much of the theoretical CL literature
\citep[e.g.,][]{evron2022catastrophic,evron23continualClassification,swartworth2023nearly,goldfarb2024theJointEffect,jung2025convergence}.
This assumption simplifies the analysis\textsuperscript{\ref{fn:setting}} and rules out cases where forgetting previous tasks is {beneficial}, as new tasks may directly contradict them.
Finally, this assumption is very reasonable in highly overparameterized models, \eg deep networks in the neural tangent kernel (NTK) regime \citep{jacot2018ntk}.

\begin{assumption}[Joint Linear Realizability of Training Data]
\label{asm:realizability}
We assume the set of offline solutions that solve {all} tasks is nonempty.
That is, 
$\teachers
\triangleq
\Big\{
\w\in\reals^{d}
~\Big|~
\X_m \w = \y_m,
~
\forall m \in \cnt{T}
\Big\}
\neq\varnothing
\,
.
$
\end{assumption}

To facilitate the results and discussions in our paper, we focus on the offline solution with minimal norm, often associated with good generalization capabilities.

\begin{definition}[Minimum-Norm Offline Solution]
We denote, 
$\displaystyle
\teacher
\triangleq
{\argmin}_{\w\in\teachers}
\norm{\w}$.
\end{definition}


Commonly in continual learning setups, the model performance on past tasks degrades, sometimes significantly, even in linear models \citep{evron2022catastrophic}. 
Our goal is to bound this degradation, 
\ie ``forgetting". 
Following common definitions \citep[e.g.,][]{doan2021NTKoverlap,evron23continualClassification}, we define forgetting as the average increase in the loss of the \emph{last} iterate on previous tasks.

\begin{definition}[Forgetting of Training Data]
\label{def:forgetting}
    Let $\w_1,\dots,\w_k$ be the iterates of \cref{proc:regression_to_convergence} under a task ordering $\tau$.
    The forgetting at iteration $k$ is the average increase in the training loss of previously seen tasks.
    In our realizable setting, the forgetting becomes an {in-sample} loss.
    Formally,
    $$
    F_{\tau}(k) 
    =
    \frac{1}{k}\sum_{t=1}^{k} 
    \bigprn{\loss_{\tau(t)}(\w_{k})
    -
    \underbrace{\loss_{\tau(t)}(\w_{t})}_{=0}
    }
    =
    \frac{1}{2k}
    \sum_{t=1}^{k} 
    \norm{\X_{\tau(t)}\w_{k} - \y_{\tau(t)}}^2
    \,.
    $$ 
\end{definition}


Under arbitrary orderings, \citet{evron2022catastrophic} showed 
forgetting can be ``catastrophic'' in the sense that 
$\displaystyle\lim_{k\to \infty}
\mathbb{E}\left[F_{\tau}(k)\right] > 0
$.
However, as we show, this \emph{cannot} happen under random orderings.


\begin{remark}[Forgetting vs.~Regret]
While regret and forgetting are related, they can differ significantly \citep{evron2022catastrophic}.  
Regret is a key quantity in online learning, 
defined in our setting as  
$
\frac{1}{2k} \tsum_{t=1}^{k} \tnorm{\X_{\tau(t)}\w_{t-1} - \y_{\tau(t)}}^2$.
That is, it measures the suboptimality of each iterate on the \emph{consecutive} task.
In contrast, forgetting evaluates an iterate's performance across \emph{earlier} tasks.
\end{remark}

We further define the average training loss to easily connect with \emph{other} fields, such as Kaczmarz.
\begin{definition}[Training Loss]
\label{def:training_loss}
The training loss of any vector $\w\in\reals^{d}$ is given by,
    \begin{align*}
    \Loss(\w) &=
    \frac{1}{T}
    \sum_{m=1}^{T} 
    \mathcal{L}_{m} (\w)
    =
    \frac{1}{2T}
    \sum_{m=1}^{T} 
    \norm{\X_{m}\w - \y_{m}}^2
    \,.
    \end{align*}
\end{definition}
%
%
%
We bound both the forgetting and the loss, leveraging a key property---expected (in-sample) forgetting can be upper bounded using expected training loss across all tasks.
Specifically, \cref{lem:in_to_all_sample} (in \cref{app:auxiliary}) states that
$\mathbb{E}_{\tau}
[F_{\tau}(k)]
\le 
2\mathbb{E}_{\tau}\left[\Loss\left(\w_{k-1}\right)\right]
+
\frac{\norm{\teacher}^{2}R^{2}}{k}$
in orderings with replacement,
where 
$R\triangleq \max_{m\in \cnt{T}} \norm{\X_m}$ is the data ``radius''
and the dependence of $\w_{k-1}$ on $\tau_{1},\dots,\tau_{k-1}$ is implicit.
Without-replacement orderings yield a related but more refined bound.
The additive $\frac{\norm{\teacher}^{2}R^{2}}{k}$ term is negligible compared to other terms in our bounds.
%

\pagebreak

\section{Reductions: From Continual Linear Regression to Kaczmarz to SGD}
\label{sec:reductions}

Prior work has drawn connections between continual linear regression and the Kaczmarz method \citep{evron2022catastrophic},
which we revisit pedagogically to keep the paper self-contained.
Importantly, this leads us to a {novel} reduction between the (block) Kaczmarz method and SGD on special functions (\cref{scheme:kaczmarz,scheme:sgd}).
As illustrated in \cref{fig:regression_reductions}, this analytical flow
allows us to improve the rates for continual and Kaczmarz methods
by analyzing the last iterate of SGD instead.


\noindent
\begin{minipage}{0.475\textwidth}
{\centering
\begin{algorithm}[H]
   \caption{ \strut The Block Kaczmarz Method
    \label{scheme:kaczmarz} \strut}
\begin{algorithmic}
   \STATE {\algmargin
   \textbf{Input:}
   Jointly realizable ${(\X_{m}, \y_{m}), 
    \forall m\!\in\!\cnt{T}}$
   }
   \STATE {
   \algmargin Initialize} 
   $\w_0 = \0_{\dim}$
   \STATE {
   \algmargin 
   For each iteration $t=1,\dots,k$:
   }
   \STATE {
   \algmargin \hspace{1em}
   $\w_t
   \leftarrow
   \w_{t-1} -
   \X_{\tau(t)}^{+}
   \tprn{\X_{\tau(t)}\w_{t-1}-\y_{\tau(t)}}$
   }
\end{algorithmic}
\end{algorithm}
}
\end{minipage}
\hfill
\begin{minipage}{0.515\textwidth}
{\centering
\begin{algorithm}[H]
   \caption{ \strut SGD with $\eta=1$ on special $\{f_m\}_{m}$%
   \strut
    \label{scheme:sgd}}
\begin{algorithmic}
    \vspace{-0.1em}
   \STATE {\algmargin
   \textbf{Input:}
   ${f_{m}(\w) \!=\! 
    \frac{1}{2}\!\norm{\X_{m}^+\prn{\X_{m}\w \!-\! \y_{m}}}^{2}\!, 
    \forall m\!\in\!\cnt{T}}$
   }
   \STATE {\algmargin Initialize
   $\w_0 = \0_{\dim}$}
   \STATE {
   \algmargin 
   For each iteration $t=1,\dots,k$:
   }
   \STATE {
   \algmargin \hspace{1em}
   $\w_t
   \leftarrow
    \w_{t-1} 
    -
    \nabla_{\w}
    f_{\tau(t)} \bigprn{\w_{t-1}}$ 
   } 
\end{algorithmic}
\end{algorithm}
}
\end{minipage}

\vspace{-1pt}

\subsection{Revisit: Continual Linear Regression and the Kaczmarz Method}
\label{sec:kaczmarz_reduction}

The (block) Kaczmarz method in \cref{scheme:kaczmarz} \citep{karczmarz1937angenaherte,elfving1980block} is a classical iterative method for solving a linear system \( \X \w = \y \), easily mapped to our learning problem by stacking tasks in blocks, \ie
$$\X=\begin{pmatrix} 
	\X_{1} \vspace{-0.35em}     \\
	\scalebox{0.85}{$\vdots$} \\
	\X_{T} 
\end{pmatrix}
\in \reals^{N \times d},
\quad
\y=\begin{pmatrix} 
	\y_{1} \vspace{-0.3em}    \\
	\scalebox{0.85}{$\vdots$} \\
	\y_{T} \end{pmatrix}
\in \reals^{N},
\quad
\text{where $N = \sum_{m=1}^{T} n_m$.}
$$
In each iteration, the Kaczmarz method (\cref{scheme:kaczmarz}) perfectly solves the current block, 
\linebreak
\ie ${\X_{\tau(t)}\w_{t}=\y_{\tau(t)}}$ 
{(to see that, recall that $\X_{\tau(t)}^{+}$ denotes the Moore-Penrose pseudo-inverse of $\X_{\tau(t)}$)}.
The continual \cref{proc:regression_to_convergence} also minimizes the current loss {to convergence}, 
\ie until it is perfectly solved (in the realizable case).
In fact, \citet{evron2022catastrophic} identified the following reduction.

\begin{reduction}[Continual Regression $\Rightarrow$ Block Kaczmarz]
\label{reduc:cl_to_kaczmarz}
In the realizable case (\cref{asm:realizability}) under any ordering $\tau$, 
continual linear regression learned to convergence%
\footnote{
The learner minimizes $\loss_{\tau(t)}$ with (S)GD to convergence;
the pseudo-inverse is \emph{not} computed explicitly.
} is equivalent to the block Kaczmarz method.
That is, the iterates $\w_0,\dots,\w_k$ 
of \cref{proc:regression_to_convergence,scheme:kaczmarz} coincide.
\end{reduction}


\subsection{New Reduction: 
Kaczmarz Method and Stepwise-Optimal Stochastic Gradient Descent}
\label{sec:kaczmzarz_to_sgd}

\paragraph{Rank-$1$ data.}
It is known that when each task contains \emph{just one} row, each update in the Kaczmarz method corresponds to a gradient step on 
with a specific ``normalizing'' step size \citep{needell2014stochastic}.
That is, since in rank-$1$ we have
$
   \loss_{\tau(t)} (\w)
   =
   \frac{1}{2}
   \bignorm{\x_{\tau({t})}^\top \w - y_{\tau({t})}}^2
$,
Kaczmarz updates hold
\begin{align}
\label{eq:rank1_sgd}
\w_{t}
=
\w_{t-1} -
\tfrac{1}{\tnorm{\x_{\tau(t)}}^2}
\bigprn{
    \x_{\tau(t)}^{\top}\w_{t-1}
    -
    y_{\tau(t)}
}
\x_{\tau(t)}
=
\w_{t-1} -
\tfrac{1}{\tnorm{\x_{\tau(t)}}^2}
\nabla_{\w}
\loss_{\tau(t)} (\w_{t-1})
\,.
\end{align}

\paragraph{What about {higher} data ranks?}

We now establish a more general reduction from the {block} Kaczmarz method---at \emph{any} rank---to SGD
(in \secref{sec:extensions}, we similarly connect SGD and the broader Projection Onto Convex Sets framework, 
extending our results to continual linear \emph{classification}).

\begin{reduction}[Block Kaczmarz $\Rightarrow$ SGD]
\label{reduc:kaczmarz}
In the realizable case (\cref{asm:realizability}) under any ordering $\tau$, 
the block Kaczmarz method
is equivalent to SGD with a step size of $\eta=1$,
applied w.r.t.~a convex,
$1$-smooth least squares objective:
$\big\{f_{m}(\w) \triangleq
\frac{1}{2}\norm{\X_{m}^+\prn{\X_{m}\w - \y_m}}^{2}
\big\}_{m=1}^{T}$.
That is, the iterates $\w_0,\dots,\w_k$ 
of \cref{scheme:kaczmarz,scheme:sgd} coincide.
\end{reduction}
Intuitively, the $\X_{m}^+$ term in the modified objectives $\{f_{m}\}$
generalizes the normalizing step size from the rank-1 case, 
fitting all directions in the current block precisely with the same step size $\eta=1$.


The reduction above is key to our analysis flow (\cref{fig:regression_reductions}) as it reveals that 
continual linear regression can be analyzed directly via SGD analysis.
It follows from substituting the gradient from the next lemma into $\tprn{\w_{t-1} - \nabla_{\w} f_{\tau(t)} (\w_{t-1})}$ in \cref{scheme:sgd}.
The lemma is proved in \cref{app:auxiliary}.

\vspace{-1pt}

\begin{lemma}[Properties of the Modified Objective]
\label{lem:cl_gd_equiv}
Consider any realizable task collection 
s.t.\linebreak
${\X_m\teacher=\y_m}, \forall m\in\cnt{T}$.
Define 
${f_{m}(\w) = 
\frac{1}{2}\norm{\X_{m}^+\prn{\X_{m}\w - \y_m}}^{2}}$.
Then, $\forall m\in \cnt{T}, \w\in\R^{d}$,
\begin{enumerate}[label=(\roman*), leftmargin=*,itemsep=1pt,topsep=4pt]
\item {Upper bound:} 
$\mathcal{L}_{m}(\w)
\le R^2 f_{m} (\w)
\triangleq
\max_{m'\in \cnt{T}} \norm{\X_{m'}}^2 f_{m} (\w)
$\,.

\item {Gradient:}
\hspace{1.5em}
$\nabla_{\w} f_m(\w) = \X_{m}^+\X_{m}\w - \X_{m}^{+}\y_m$\,.

\item {Convexity and Smoothness:}
$f_m$ is convex and $1$-smooth.

\end{enumerate}
\end{lemma}
%

\vspace{-7pt}

\section{Rates for Random-Order Continual Linear Regression and Kaczmarz}
\label{sec:rates}

This section improves the best known upper bound: for random orderings with replacement,
\citet{evron2022catastrophic} proved a forgetting rate of
$\mathbb{E}_{\tau}
\left[
F_{\tau}(k)
\right]
=
\bigO\!\prn{\frac{d-\rankavg}{k}
}$
where $\rankavg\triangleq\frac{1}{T}\sum_{m} \rank(\X_m)$.
Notably, their rate depends on the dimensionality $d$, challenging the transfer of insights from linear models to highly overparameterized deep networks (\eg via the NTK regime).
Encouragingly, they only provided a worst-case lower bound of $1/k$, calling for further research to narrow this gap.

We tighten the existing problem-dependent rate from $\bigprn{d-\rankavg}$ to $\min\bigprn{\sqrt{d - \rankavg},\sqrt{T\rankavg}}$, 
and prove a problem-\emph{independent} rate of ${1/\sqrt[4]{k}}$.
Finally, we provide the first rates for \emph{without}-replacement orderings, isolating the effect of randomness versus repetition.
See summary in the table below.

\begin{table*}[ht]
\centering
\captionsetup{font=small}
\caption{
\textbf{Forgetting and Loss Rates in Continual Linear Regression (and Block Kaczmarz).}
\linebreak
Upper bounds apply to any $T$ realizable tasks (or blocks).
%
Lower bounds indicate {worst cases}, 
\ie specific constructions.
Random ordering bounds apply to the {expected} forgetting (or loss).
We omit mild constant multiplicative factors and an unavoidable $\norm{\teacher}^{2\!} R^2$ term.
Finally, ${a\wedge b\triangleq\min(a,b)}$.
\linebreak
Recall: $k=\,$iterations; 
$d=\,$dimensionality;
$\rankavg,r_{\max}=\,$average and maximum data matrix ranks.
\label{tab:comparison}
}
\vskip -0.2cm
\small
\begin{tabular}{c c m{0.22\textwidth} m{0.22\textwidth} c} 
\Xhline{1.2pt}
\rule[-11pt]{0pt}{24pt}
\makecell{\vspace{-0.55em}\\ 
\textbf{Paper / Ordering}} & 
\makecell{\vspace{-0.55em}\\ 
\textbf{Bound}} &
\centering\makecell{\small\textbf{Random} \\ 
\textbf{with Replacement}}
& 
\centering\makecell{\small\textbf{Random} \\ 
\textbf{w/o Replacement}}
& 
\makecell{\vspace{-0.55em}\\ \textbf{Cyclic}} 
\\ 
\Xhline{1.2pt}
\rule[-14pt]{0pt}{31pt}
\!\citet{evron2022catastrophic}\!
&
Upper
& \centering
$\displaystyle
\frac{d-\rankavg}{k}$     
&
\centering---
&
\!\!\!\!\!\!
$\displaystyle
\frac{T^2}{\sqrt{k}} 
\wedge
\frac{T^2 (d-r_{\max})}{k}
$
\!\!\!\!
\\
\rule[-12pt]{0pt}{25pt}
\makecell{\small\citet{swartworth2023nearly}}
&
Upper
& \centering---      & \centering---      
& $\displaystyle
\frac{T^3}{k}
$      
\\
\rule[-16pt]{0pt}{29pt}
\textbf{Ours}
&
Upper
& \centering
$\displaystyle
\frac{1}{\sqrt[4]{k}} \wedge
\frac{\sqrt{d-{\rankavg}}}{k} \wedge
\frac{\sqrt{T\rankavg}}{k}
$
& \centering
$\displaystyle
\frac{1}{\sqrt[4]{T}} \wedge
\frac{d-{\rankavg}}{T}$      
& ---      
\\ 
\Xhline{1.2pt} 
\rule[-14pt]{0pt}{29pt}
\makecell{\vspace{-0.8em}\\ 
\citet{evron2022catastrophic}}
&
\makecell{\vspace{-0.8em}\\
Lower}
& \centering
$\displaystyle
\frac{1}{k}$ \texttt{(*)}
& \centering
$\displaystyle
\frac{1}{T}$ \texttt{(*)} & 
\makecell{\vspace{-0.95em}\\ 
$\displaystyle
\frac{T^2}{k}$}
\\ 
\Xhline{1.2pt}
\end{tabular}

\vskip 0.1cm
\caption*{
\small
\texttt{(*)} 
They did not explicitly provide such lower bounds, but
the $2$-task construction from their proof of {Theorem~10}, can yield a $\Theta(\hfrac{1}{k})$ random~behavior by cloning those $2$ tasks $\floor{T/2}$ times for any general $T$.}
\vspace{-22pt}
\end{table*}

\subsection{A Parameter-Dependent $\bigO(1/k)$ Rate}
\label{sec:parameter-dependent}
Here, we present a tighter $\sqrt{d-\rankavg}$ term
and a term depending only on the rank and number of tasks.

\begin{theorem}[Parameter-Dependent Forgetting Rate for Random With Replacement]
\label{thm:random_convergence_rate}
Under a random ordering with replacement 
over $T$ jointly realizable tasks, 
the expected loss and forgetting of {Schemes~\ref{proc:regression_to_convergence},~\ref{scheme:kaczmarz}} after \( k \geq 3\) iterations are bounded as,
\begin{align*}
\mathbb{E}_{\tau}\!\left[\Loss\left(\w_{k}\right)\right]
&
\le
\frac{\min\prn{\sqrt{d - \rankavg},\sqrt{T\rankavg}}
\tnorm{\teacher}^2 R^2}{
2e(k - 1)}
,\,\,\,
\mathbb{E}_{\tau}\!
\left[
F_{\tau}(k)
\right]
\leq 
\frac{3\min\prn{\sqrt{d - \rankavg},\sqrt{T\rankavg}}
\tnorm{\teacher}^2 R^2
}{
2\prn{k - 2}}
,
\end{align*}
where $\rankavg\triangleq\frac{1}{T}\sum_{m \in [T]} \rank(\X_m)$.
(Recall that ${R\triangleq\max_{m \in [T]} \norm{\X_m}}$.)
\end{theorem}
%

%
Our proof, given in \appref{app:parameter-dependent}, is related to a recent work \citep{guo2022rates} that characterizes the weak error 
(similar to our loss) by analyzing a linear map.
Unlike ours, the polynomial rates they derive involve matrix properties related to the condition number.

\paragraph{Proof Idea.}
We rewrite the Kaczmarz update (\cref{scheme:kaczmarz}) in a recursive form of the differences, \ie $\w_t - \teacher = \mP_{\tau(t)}\prn{\w_{t-1} - \teacher}$ for a suitable projection matrix $\mP_{\tau(t)}$.
We define the linear map\linebreak $Q\sqprn{\A} = \frac{1}{T}\sum_{m=1}^{T} \mP_m \A \mP_m$ to capture the evolution of the difference's second moments, enabling sharp analysis of the expected loss in terms of $Q$. 
Using properties of $Q$, norm inequalities, and the spectral mapping theorem, we establish a fast $\bigO \prn{1/k}$ rate with explicit dependence on $T$, $d$, and $\rankavg$.

\begin{remark}[The $\norm{\teacher}^{2}\!R^{2}$ Scaling Term]
All the rates we derive contain a multiplicative factor of $\|\teacher\|^{2} R^{2}$, 
a generally unavoidable scaling term in linear regression. 
Prior work on continual learning has either normalized it away implicitly—e.g., by 
assuming $\norm{\teacher}^{2}, R \le 1$~\citep{evron2022catastrophic}—or included it 
explicitly, as we do~\citep{evron23continualClassification,lin2023theory}.
The rate in \cref{thm:random_convergence_rate} involves additional problem 
parameters, \ie $T$, $d$, and $\rankavg$, whereas the rate in 
\cref{thm:cl_by_sgd_main} below is “universal’’ in the sense that it does not 
depend on any such parameter. 
\end{remark}

\subsection{A Universal $\bigO(1/\sqrt[4]{k})$ Rate}
\label{sec:universal_by_sgd}
Next, we present a forgetting rate \emph{independent} on the dimensionality, rank, and number of tasks.
This independence is crucial in highly overparameterized regimes, as encountered in deep neural networks.

\begin{theorem}[Universal Forgetting Rate for With-Replacement Random Ordering]
\label{thm:cl_by_sgd_main}
Under a random ordering with replacement 
over $T$ jointly realizable tasks, 
the expected loss and forgetting of {Schemes~\ref{proc:regression_to_convergence},~\ref{scheme:kaczmarz}}
    after $k\geq 2$ iterations are bounded as,
\begin{align*}
\mathbb{E}_{\tau}\!\left[\Loss\left(\w_{k}\right)\right]
&
\le
\frac{2\norm{\teacher}^2 R^2}{ \sqrt[4]{k}}
\,,\quad\quad
\mathbb{E}_{\tau}\!
\left[
F_{\tau}(k)
\right]
\le 
\frac{5\norm{\teacher}^2 R^2}{ \sqrt[4]{k-1}}
\,.
\end{align*}
\end{theorem}

We prove this result in \cref{sec:cl_by_sgd_proof} by leveraging the connections between CL and SGD.
Specifically, \cref{sec:reductions} showed that continual linear regression is equivalent to SGD with step size \emph{exactly} $1$ on a related least squares objective that bounds the original continual learning loss.
Our result then follows from our novel last-iterate SGD bounds that, crucially, apply even to this specific step size.
To ease readability, we keep a CL perspective here and defer last-iterate analysis to \cref{sec:sgd}.

\pagebreak

\subsection{Random Task Orderings Without Replacement}
\label{sec:cl_wor}

\citet{evron2022catastrophic} suggested that forgetting is `catastrophic' only when 
$
\lim_{k\to \infty}
\mathbb{E}\left[F_{\tau}(k)\right] > 0$,
and presented such an adversarial case with a deterministic task ordering where $k=T\to\infty$.
\linebreak
In contrast, they showed that cyclic or random task orderings mitigate forgetting, perhaps due to task repetition.
So far, under random orderings, it has been difficult to disentangle the effect of randomness from that of repetition---\ie whether their remedying impact arises from random permutation or repeated exposure.
Below, we provide the first result demonstrating that randomly permuting tasks is sufficient to alleviate catastrophic forgetting.

\begin{theorem}[Forgetting Rates for Without-Replacement Random Ordering]
\label{thm:cl_by_sgd_wor}
Under a random ordering without replacement 
over $T$ jointly realizable tasks, 
the expected loss and forgetting of
{Schemes~\ref{proc:regression_to_convergence},~\ref{scheme:kaczmarz}} after $k\in \{2,\dots,T\}$ iterations are both bounded as,
    \begin{align*}
        \mathbb{E}_{\tau}\!
        \left[\Loss\left(\w_{k}\right)\right],
        ~
        \mathbb{E}_{\tau}\!\sbr{F_{\tau}(k)}
        \leq 
        \min\left(
        \frac{7}{\sqrt[4]{k-1}},\,
        \frac{d-\rankavg+1}{k-1}
        \right)
        {
        \norm{\teacher}^2 R^2}
        .
    \end{align*}
\end{theorem}
The proof is given in \cref{sec:cl_wor_proofs}.
The dimensionality-dependent term parallels the with-replacement case in App.~D.1.2 of \citet{evron2022catastrophic}, but requires a refined upper bound on in-sample forgetting.
The dimensionality-independent term again relies on last-iterate analysis, as presented in \cref{app:sgd_without}.

In \cref{sec:related_works}, we discuss connections between our result above and areas like shuffle SGD.

\section{Last-Iterate SGD Bounds for Linear Regression}
\label{sec:sgd}
In this self-contained section, we derive last-iterate guarantees for SGD in the realizable stochastic least squares setup.
Motivated by the connection with continual regression discussed in \cref{sec:reductions}, we focus on regression problems that are $\beta$-smooth individually, and obtain upper bounds for the last SGD iterate that apply for a significantly wider range of step sizes compared to prior art \citep{varre2021last}. 
Notably, this is the first time convergence of SGD in this setup is established for a range of step sizes completely independent of the optimization horizon.
\cref{table:wr_sgd} in \cref{sec:related_works} compares our bounds with related work and classical results in the field.

Recent work has analyzed SGD in \emph{realizable} (possibly noisy) least squares settings \citep{ge2019stepDecay,vaswani2019fast,berthier2020tight,zou2021benign,varre2021last,wu2022last}. 
Realizable settings are primarily motivated by connections to deep networks in the overparameterized regime \citep{ma2018power}, where models are expressive enough to perfectly fit the training data.
With the exception of \citet{varre2021last}, most of these works focus on non-fixed step sizes and/or provide guarantees for the average iterate
(see \cref{sec:related_works} for discussion).
Similarly, here we study the following stochastic, jointly realizable least squares problem.

\begin{setup}\label{setup:sgd_main}
    Let $\cI$ be an index set, 
    and $\cD$ a distribution over $\cI$. 
    We consider the optimization objective:
    \begin{align*}
        {\min}_{\w\in \R^d}\cbr{\,
        \f(\w) 
        \eqq 
        \E_{i\sim \cD}
        f(\w; i)
        \eqq 
        \E_{i\sim \cD}
        \left[\tfrac12\norm{\A_i \w - \vb_i}^2
        \right]
        \,
    },
    \end{align*}
    %
    %
    \noindent
    where
    $\A_i \in \R^{n_i\times d}, \vb_i \in \R^{n_i},~\forall i\in \cI$.
    We specifically focus on {$\beta$-smooth} functions, 
    that is, \linebreak ${\norm{\A_i^\top \A_i}\leq \beta},\forall {i\in \cI}$,
    under a {realizable} assumption, 
    \ie $\exists\teacher\in \R^d: \f(\teacher)=0$.
\end{setup}

Our main result establishes last-iterate guarantees for with-replacement SGD, defined next.
Given an initialization $\w_0\in \R^d$ and step-size $\eta>0$:
\begin{align}\label{def:sgd_withreplacement}
    \w_{t+1} \gets \w_t - \eta \nabla f(\w_t; i_t),
    \quad i_t \sim \cD.
\end{align}

\pagebreak
Below, we state our theorem and then provide an overview of the analysis.
\begin{theorem}[Last-Iterate Bound for Realizable Regression With Replacement]
\label{thm:sgd_last_iterate_main}
    Consider the \linebreak {$\beta$-smooth}, realizable \cref{setup:sgd_main}.
    Then, for any initialization $\w_0 \in \R^d$, 
    with-replacement SGD (\cref{def:sgd_withreplacement})
    with step size $\eta < 2/\beta$,
    holds:
$$
\E \f(\w_T) 
\leq 
\frac{e D^2}{2\eta (2-\eta\beta)T^{1-\eta\beta\left(1-\eta\beta/4\right)}}\,,\quad
\forall T\ge 1
\,,
$$
where $D\eqq\norm{\w_0 - \teacher}$. 
In particular, for $\eta=\frac{1}{\beta}$,
$\E \f(\w_T) 
    \leq \frac{e\beta D^2}{2\sqrt[4]{T}}$.
\end{theorem}

An important part of \cref{thm:sgd_last_iterate_main} is the $(2-\eta\beta)$ factor in the denominator, replacing a $(1-\eta\beta)$\linebreak common in standard analysis.
This difference makes our theorem applicable to the continual regression setting which requires setting $\eta=1/\beta$ (\cref{reduc:kaczmarz}).
In addition, 
 for $\eta=\hfrac{1}{\prn{\beta\log T}}$, we recover the near-optimal rate obtained by \citet{varre2021last},
 \ie 
$
\E \f(\w_T) 
=\bigO\prn{\frac{\beta D^2 \log T}{T}}.$

\paragraph{Extension to Without-Replacement SGD.}
In \cref{app:sgd_without}, we extend \cref{thm:sgd_last_iterate_main} to SGD without replacement.
The proof leverages algorithmic stability for SGD \citep{bousquet2002stability,shalev2010learnability,hardt2016train}, focusing on a variant tailored to without-replacement sampling \citep{sherman2021optimal,koren2022benign}.
In particular, we establish a new bound for this variant in the smooth and realizable regime, which has not appeared in prior work.

\paragraph{Analysis Overview.}
Here, we briefly outline the proof of \cref{thm:sgd_last_iterate_main}, 
which follows immediately by combining the two lemmas below
(while noting that $\eta < 2/\beta \Rightarrow
e^{\eta\beta\left(1-\eta\beta/4\right)}\le e$).
The first step of the proof is to establish a regret bound for SGD when applied to $f(\w;i_1)\ldots f(\w;i_T)$, 
holding for any step size $\eta<2/\beta$. 
This already departs from the standard $\eta<1/\beta$ mandated by standard analysis.
All proofs for this section are given in \cref{app:sgd_with}.
\begin{lemma}[Gradient Descent Regret Bound for Smooth Optimization]
\label{lem:regret_bound} 
Consider the $\beta$-smooth, realizable \cref{setup:sgd_main},
and let $T\geq 1$, $(i_0, \ldots, i_{T}) \in \cI^{T+1}$ be an arbitrary sequence of indices in $\cI$, 
and ${\w_0\in \R^d}$ 
be an arbitrary initialization.
Then, 
the gradient descent iterates given by
${\w_{t+1} \gets \w_t - \eta \nabla f(\w_t; i_t)}$
for a step size $\eta<2/\beta$,
hold:
    \begin{align*}
        \sum_{t=0}^T f\left(\w_t; i_t\right)
        \leq 
        \frac{\norm{\w_0 - \teacher}^2}{2\eta(2-\eta\beta)}
        \,.
    \end{align*}
\end{lemma}
The second and main step of the analysis is to relate the loss of the last SGD iterate to the regret of the algorithm. 
For this, we carefully adapt an existing approach for last-iterate convergence in the non-smooth case \citep{shamir13sgd}. 
The result, given below, is slightly more general to accommodate without-replacement sampling, addressed in the next section.
\begin{lemma}
\label{lem:ratio_avg_last}
Consider the $\beta$-smooth, realizable \cref{setup:sgd_main}. 
Let $T\geq 1$.
Assume $\cP$ is a distribution over $\cI^{T+1}$ such that for every $0\leq t\leq\tau_1 \leq \tau_2 \leq T$, the following holds:
For any ${i_0,\ldots i_{t-1}\in\cI^{t}}, i\in \cI$, 
$\Pr(i_{\tau_1}=i|i_0,\ldots,i_{t-1})=\Pr(i_{\tau_2}=i|i_0,\ldots,i_{t-1})$.
Then, for any initialization $\w_0\in \R^d$, 
with-replacement SGD (\cref{def:sgd_withreplacement}) with step-size $\eta<2/\beta$, holds:
    \begin{align*}
     \E f(\w_T,i_T)\leq (eT)^{\eta\beta\left(1-\eta\beta/4\right)}
     \E \left[
     \frac{1}{T+1}\sum_{t=0}^T f(\w_t;i_t)\right],
    \end{align*}
    where the expectation is taken with respect to $i_0,\ldots,i_T$ sampled from $\cP$.
\end{lemma}


\pagebreak

\section{Extensions}
\label{sec:extensions}

\subsection{A Universal $\bigO(1/\sqrt[4]{k})$ Rate for General Projections Onto Convex Sets}

Projections Onto Convex Sets (POCS) is a classical method
that iteratively projects onto closed convex sets to find a point in their intersection
\citep{gubin1967methodProjections,boyd2003alternating}.
Formally, 

{\centering
\begin{algorithm}[H]
   \caption{Projections onto Convex Sets (POCS)
    \label{scheme:pocs}
    }
\begin{algorithmic}
   \STATE {\algmargin
   \textbf{Input:} 
   A set of $T$ closed convex sets $\convexset_{1},\dots,\convexset_{T}$;\quad
   an initial $\w_{0}\in\R^{d}$;\quad 
   an ordering $\tau:\cnt{k}\to\cnt{T}$
   }
   \STATE {
   \algmargin 
   For each iteration $t=1,\dots,k$:
   }
   \STATE {
   \algmargin \hspace{1em}
   $\w_t \leftarrow \proj_{\tau(t)} (\w_{t-1})
   \triangleq 
   \argmin_{\w \in \convexset_{\tau(t)}}
   \norm{\w-\w_{t-1}}
   $
   }
   \vspace{-0.1em}
\end{algorithmic}
\end{algorithm}
%
}
\vspace{-0.2em}

Generalizing \cref{reduc:kaczmarz} (Kazcmarz$\Rightarrow$SGD),
we note that POCS algorithms also implicitly perform stepwise-optimal SGD w.r.t.~a convex, $1$-smooth least squares objective.\footnote{
This has been \emph{partially} observed in the POCS literature
\citep[e.g.,][]{nedic2010randomPOCS}.
}
Proofs for this section are given in \cref{app:extensions_proofs}.
\begin{reduction}[POCS $\Rightarrow$ SGD]
\label{reduc:pocs}
Consider $T$ arbitrary (nonempty) closed convex sets $\convexset_{1},\dots,\convexset_{T}$,
initial point $\w_0\in\R^{d}$, and ordering $\tau$.
Define $f_{m}(\w) = 
\frac{1}{2}\norm{\w - \proj_{m}(\w)}^{2}, \forall m\in\cnt{T}$.
Then,
\vspace{-0.1em}
\begin{enumerate}[label=(\roman*), itemindent=-0.3cm, labelsep=0.2cm, itemsep=4pt,topsep=3pt]
    \item $f_m$ is convex and $1$-smooth.
    \item The POCS update is equivalent to an SGD step:
    $
\w_t 
=
\proj_{\tau(t)} (\w_{t-1})
=
\w_{t-1} 
-
\nabla_{\w}
f_{\tau(t)} (\w_{t-1})$.
\end{enumerate}
\end{reduction}

We can now employ our analysis from \cref{sec:sgd} to yield a universal rate.
\begin{theorem}[Universal POCS Rate]
\label{thm:pocs-rate}
Consider the conditions of \cref{reduc:pocs} and assume a \linebreak
nonempty intersection
${\convexintersection
=
\bigcap_{m=1}^{T} \convexset_{m}
\neq
\varnothing}
$.
Then, under a random ordering with or without replacement, 
the expected ``residual'' of \cref{scheme:pocs} after
$\forall k \ge 1$ iterations ($k\in\cnt{T}$ without replacement) is bounded as,
\begin{align*}
\mathbb{E}_{\tau}
\bigg[
\frac{1}{2T}
\sum_{m=1}^{T} 
\norm{\w_{k} - \proj_{m}(\w_{k})}^{2}
\bigg]
=
\mathbb{E}_{\tau}
\bigg[
\frac{1}{2T}
\sum_{m=1}^{T} 
\mathrm{dist}^2 (\w_{k}, \convexset_{m})
\bigg]
\le 
\frac{7}{\sqrt[4]{k}}
\min_{\w\in \convexintersection}
\norm{\w_{0} - \w}^2\,.
\end{align*}
\end{theorem}
To the best of our knowledge, this is the first universal rate in the POCS literature, independent of problem parameters such as regularity or complexity, as demonstrated in \cref{sec:classification}. 
Universal rates are only achievable when analyzing individual distances,
\ie 
$f_m(
w)
=
\mathrm{dist}^2(\w, \convexset_{m})
= \norm{\w - \proj_{m}(\w)}^{2}$, rather than the distance to the intersection, \ie 
$\mathrm{dist}^2(\w, \convexintersection)$.
In machine learning, squared distances from individual sets relate not only to MSE, but also to losses such as the squared hinge loss for classification \citep{evron23continualClassification},
naturally leading to our next continual model.

\begin{figure}[h!]
    \centering
    \includegraphics[width=0.99\linewidth]{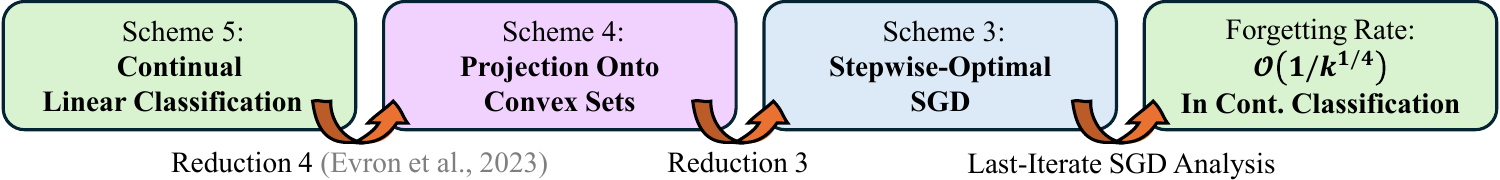}
    \caption{Analysis Flow Leading to Our Improved Classification Rates.
    }
    \vspace{-1em}
    \label{fig:classification_reductions}
\end{figure}

\pagebreak

\subsection{A Universal $\bigO(1/\sqrt[4]{k})$ Rate for Random Orderings in Continual Linear Classification}
\label{sec:classification}

Regularization methods are commonly used to prevent forgetting in CL \citep[see][]{kirkpatrick2017overcoming,aljundi2018memory,li2023fixed}.
\citet{evron23continualClassification} studied a weakly-regularized linear model for continual classification.
They considered $T\ge 2$ jointly separable, binary classification tasks,
defined by datasets $\dataset_1, \dots, \dataset_{T}$ consisting of vectors ${\x\in \reals^{\dim}}$ and their labels $y\in \left\{-1,+1\right\}$.

\noindent
{\centering
\begin{algorithm}[H]
   \caption{ \strut Weakly-Regularized Continual Linear Classification (for $\lambda\to 0$)
    \label{scheme:regularized_cl} \strut}
\begin{algorithmic}
   \vspace{-0.1em}
   \STATE {
   \algmargin Initialize} 
   $\w_{0}^{(\lambda)} = \0_{\dim}$
   \STATE {
   \algmargin 
   For each iteration $t=1,\dots,k$:
   }
   \STATE {
   \algmargin \hspace{.7em}
   $\displaystyle
    \w_{t}^{(\lambda)}
   \leftarrow
    \argmin_{\w \in \reals^{\dim}} 
    {
    \tsum_{(\x,y)\in \dataset_{t} }
    e^{-y \w^\top \x }
    +
    \frac{\lambda}{2}\,
    \bignorm{\w-\w_{t-1}^{(\lambda)}}^2
    }$
   }
\end{algorithmic}
\end{algorithm}
}

Specifically, \citet{evron23continualClassification} proved that learning an entire (separable) classification task in this continual scheme implicitly applies projection onto convex sets. More formally,

\begin{reduction}[Continual Classification $\Rightarrow$ POCS]
\label{reduc:classification_to_pocs}
Given jointly separable tasks, the continual iterates of \cref{scheme:regularized_cl} in the limit as \(\lambda \to 0\), align in direction of the sequential projections of \cref{scheme:pocs} onto the convex sets defined as
$\convexset_{m}\triangleq
    \left\{
    \w\in\reals^{\dim}
    \mid
    y \w^\top \x\ge 1, 
    ~
    \forall (\x, y)\in \dataset_m \right\},\,
    \forall m\in \cnt{T}$.
\end{reduction}

Importantly, this reduction enables the study of continual classification through projection algorithms.
In particular,
\citet{evron23continualClassification} studied forgetting using an equivalent of our \cref{def:forgetting}:
$$
F_{\tau}(k) 
=
\frac{1}{k}
\sum_{t=1}^{k} 
\bigprn{\loss_{\tau(t)}(\w_{k})
-
\loss_{\tau(t)}(\w_{t})}
\le 
\frac{R^2}{2k}
\sum_{t=1}^{k} 
\norm{\w_{k-1} - \proj_{\tau(t)}(\w_{k-1})}^{2}\,.
$$

As illustrated in \cref{fig:classification_reductions}, we derive the following bound by combining their reduction with our POCS rate (\cref{thm:pocs-rate}) and SGD stability arguments.
\begin{theorem}
\label{cor:pocs_forgetting}
Under a random ordering, with or without replacement, over $T$ jointly separable tasks, the expected forgetting of the weakly-regularized \cref{scheme:regularized_cl}
after $k \ge 1$ iterations is,
\begin{align*}
\mathbb{E}_{\tau}\!
\left[
F_{\tau}({k})
\right]
\le 
\frac{7\norm{\teacher}^2
R^2}{\sqrt[4]{k}}
\,,
\quad
\text{where }
\teacher \in
{\argmin}_{\w\in \convexset_1\cap\dots \cap\convexset_T} \norm{\w_{0} - \w}^2
\,.
\end{align*}
\end{theorem}

As shown in \cref{tab:classification},
our rate is universal while the previous one depends on $\norm{\teacher}^{2}R^2$, often seen as the ``complexity'' of classification problems.
For example, after $k=4T\norm{\teacher}^{2} R^2$ iterations,
the existing (normalized) rate is $e^{-1}$,
while ours is potentially much smaller:
$\frac{5}{T^{1/4}\sqrt{\left\Vert \teacher\right\Vert R}}$.

\vspace{-0.4em}
\begin{table*}[ht]
\centering
\captionsetup{font=small}
\caption{
\textbf{Forgetting Rates in Weakly-Regularized Continual Linear Classification on Separable Data.}
All cells omit mild multiplicative constants and normalize by an unavoidable $\norm{\teacher}^2 R^2$ term.\!\!\!
\label{tab:classification}
}
\small
\vskip -5pt
%
\begin{tabular}{c c c c} %
\Xhline{1.2pt}
\rule[-11pt]{0pt}{28pt}
\makecell{\vspace{-1.2em}\\ 
\small\textbf{Paper / Ordering}}
& 
\centering\makecell{\textbf{Random} \\ 
\textbf{with Replacement}}
& 
\centering\makecell{\textbf{Random} \\ 
\textbf{w/o Replacement}}
& 
\makecell{\vspace{-1.2em}\\ 
\textbf{Cyclic}} 
\\ 
\Xhline{1.2pt}
\rule[-11pt]{0pt}{28pt}
\makecell{ \citet{evron23continualClassification}}
 & \centering
$\displaystyle
\exp
\Bigprn{-\frac{k}{4T\tnorm{\teacher}^2 R^2}}
$
&
\centering ---
&
$\displaystyle
\frac{T^2}{\sqrt{k}}
\wedge
\exp
\Bigprn{-\frac{k}{16T^2\tnorm{\teacher}^2 R^2}}
$
\\
\rule[-11pt]{0pt}{28pt} 
\makecell{
\textbf{Ours}
}
& \centering
$\displaystyle
\frac{1}{\sqrt[4]{k}}
$
& \centering
$\displaystyle
\frac{1}{\sqrt[4]{T}}$      
& --- 
\\
\Xhline{1.2pt}
\end{tabular}
\vspace{-2.em}
\end{table*}



\pagebreak

\section{Discussion}
\label{sec:discussion}

Our work established reductions from continual linear regression and classification
to ``stepwise-optimal'' SGD. 
This enabled the development of analytic tools for last-iterate SGD schemes, leading to significantly improved and even universal rates for random orderings in continual learning. 
Our main results are summarized in \cref{tab:comparison,tab:classification,table:wr_sgd}.

Much of the related work has been covered throughout the paper. 
A further discussion of related work can be found in \cref{sec:related_works}.
Here, we briefly highlight additional aspects of our work.

\paragraph{Random Continual Benchmarks.}
Many popular continual benchmarks in deep learning implicitly assume a random ordering, such as the permuted MNIST benchmark \citep{kirkpatrick2017overcoming}. 
Our paper shows that in sufficiently long task sequences, 
random ordering is enough to prevent catastrophic forgetting, and the training loss goes to zero, even in the worst case. 
In accordance with our results, 
\citet{lesort2022scaling} examined a random CL benchmark---in which a subset of \emph{classes} is randomly sampled in each task---and observed that forgetting diminishes as more tasks are sampled, even while training with standard SGD (without any modifications to mitigate forgetting).
This suggests that random orderings may contaminate continual learning benchmarks, making it harder to isolate the algorithmic effects being tested.
Furthermore, real-world tasks often change gradually, not adhering to random orderings.
Such ``gradually evolving'' datasets might be more challenging and relevant as continual benchmarks.

\paragraph{Connections to the Kaczmarz Method.}
In \cref{sec:kaczmarz_reduction} we revisited known connections between continual regression and the Kaczmarz method \citep{evron2022catastrophic}.
We broadened this connection in \cref{sec:kaczmzarz_to_sgd}, bridging the \emph{block} Kaczmarz method and ``stepwise-optimal'' SGD, thus applying our novel SGD bounds to the Kaczmarz method.
Using Kaczmarz terminology, given a system $\A\x = \vb$ consisting of $T$ blocks of an average rank $\rankavg$ where 
{$\A_{m}\in\R^{{n_m}\times d}, \vb_{m}\in \R^{n_{m}}$},
our rates from \cref{sec:rates} can be summarized as 
$
\mathbb{E}_{\tau}
\big[
\frac{1}{2T} 
\sum_{m=1}^{T} \bignorm{\A_m \x_k - \vb_m}^2
\big]
=\bigO\bigprn{
\min \bigprn{
k^{-1/4}
,
\frac{1}{k}\sqrt{d - \rankavg}, \frac{1}{k}\sqrt{T\rankavg}
}
}$
for random orderings with replacement and
$\bigO\bigprn{
\min\bigprn{
k^{-1/4},
\frac{1}{k}\prn{d - \rankavg}
}
}$
without replacement.
Note that we bounded the \emph{loss}, rather than the ``error'' $\norm{\w_k - \teacher}^2$,
thus enabling the derivation of rates independent of quantities like the condition number that can make convergence arbitrarily slow.

\paragraph{Non-uniform Sampling.} The seminal work of \citet{strohmer2009randomized} proposed a Kaczmarz variant that samples rows with probability proportional to their squared norm. 
Our approach also accommodates non-uniform sampling, including norm-based ones.
Specifically, Claim~\ref{prop:mean_norm} tightens \cref{thm:cl_by_sgd_main} by employing norm-weighted \emph{block} sampling, thereby  replacing the dependence on the maximum row norm~$R$ with the average~$\bar R$.

\paragraph{From Training to Generalization Results.}
The rates derived in our paper apply to the forgetting and loss on the training sets (\cref{def:forgetting,def:training_loss}). 
These rates can be extended to the generalization loss at the cost of an additive $\bigO(1/\sqrt{N})$ term (where $N$ is the number of samples per task) via uniform-convergence arguments---or even an additive $\bigO(1/N)$ term using the more refined Rademacher bounds for linear models from \citet{srebro2010optimistic}.

\paragraph{Future and Follow-up Work.}
We narrowed the gap between existing lower and upper worst-case bounds for random orderings in continual linear regression (see \cref{tab:comparison}).
However, a considerable gap remains between
$\Omega(1/k)$ and $\bigO({1/k^{1/4}})$, largely stemming from  our proof technique in \cref{lem:ratio_avg_last}. 
Our argument follows the approach of \citet{shamir13sgd}, originally developed to control the last SGD iterate in the convex, non-smooth regime. 
When instantiated in our (smooth, realizable) setting, this approach introduces an exponential dependence on $\eta\beta$.
Because of this exponential sensitivity, employing coarse inequalities (\eg bounding a negative term by $0$) can be costly: even small constant-factor losses may effectively change the power of $T$ appearing in the final bound. 

Following our reductions (\cref{sec:reductions,sec:extensions}), improved rates for ``stepwise-optimal'' SGD rates would immediately refine the bounds for continual linear regression and classification.
Indeed, a follow-up work by \citet{attia2025fast} departs from \cite{shamir13sgd} and builds on a different technique proposed by \cite{zamani2023exact}. 
Using a more refined analysis, they establish an improved rate for SGD and, combined with our reductions, obtain a tighter upper bound of $\bigO(1/\sqrt{k})$ for continual linear regression (and classification).
Finding the exact worst-case complexity between $\Omega(1/k)$ and $\bigO({1/\sqrt{k}})$ remains an open question. 

Finally, we note that our analytical flow relies on \cref{reduc:cl_to_kaczmarz}, which in turn assumes that the continual \cref{proc:regression_to_convergence} learns each task \emph{to convergence}.
Nonetheless, a concurrent follow-up by \citet{levinstein2025optimal} follows a similar overall analytical flow, despite studying more practical continual learning schemes that do not learn to convergence---namely, those using $\ell_2$ regularization or finite step budgets.
Similarly to us, they reduce learning an entire task to a single gradient step and apply last-iterate SGD analysis to obtain convergence rates.

\acks{
We thank Edgar Dobriban (University of Pennsylvania) and Amit Attia (Tel Aviv University) for fruitful discussions.

The research of DS was funded by the European Union 
{\small (ERC, A-B-C-Deep, 101039436)}. Views and opinions expressed are however those of the author only and do not necessarily reflect those of the European Union or the European Research Council Executive Agency (ERCEA). Neither the European Union nor the granting authority can be held responsible for them. DS also acknowledges the support of the Schmidt Career Advancement Chair in AI.

The research of TK has received funding from the European Research Council (ERC) under the European Union’s Horizon 2020 research and innovation program {\small (grant agreement No.\ 101078075)}.
Views and opinions expressed are however those of the author(s) only and do not necessarily reflect those of the European Union or the European Research Council. Neither the European Union nor the granting authority can be held responsible for them.
This work received additional support from the Israel Science Foundation ({\small ISF, grant number 3174/23)}, and a grant from the Tel Aviv University Center for AI and Data Science (TAD).

NS was supported in part by the Simons Foundation and the NSF-TRIPOD Institute on Data Economics and Learning (IDEAL).
}

\bibliography{99_biblio}


\appendix
\renewcommand{\thetheorem}{\Alph{section}.\arabic{theorem}}
\renewcommand{\thelemma}{\Alph{section}.\arabic{lemma}}
\renewcommand{\thecorollary}{\Alph{section}.\arabic{corollary}}
\renewcommand{\thedefinition}{\Alph{section}.\arabic{definition}}
\renewcommand{\theproposition}{\Alph{section}.\arabic{proposition}}

\makeatletter
\@addtoreset{theorem}{section}
\makeatother

\newpage
\section{Related Work}
\label{sec:related_works}
Most of the related work is already discussed in the main body of the paper.
Here, we elaborate on several interesting connections that remain open.

\paragraph{Last-iterate Guarantees for SGD.}
For the general (non-realizable) smooth stochastic setup, the recent work  of \citet{liu2024revisiting} was the first (and only, to our knowledge) to provide upper bounds on the convergence rate of the last SGD iterate. While their bounds are applicable in the realizable setting, they require non-constant step sizes to obtain non-trivial convergence, and are therefore not useful for our purposes (see \cref{table:wr_sgd}).
Our analysis technique in \cref{sec:sgd} borrows from the work of \citet[also mentioned in \cref{table:wr_sgd}]{shamir13sgd} which, in fact, belongs to the comparatively-richer line of work on 
the non-smooth setting  \citep{shamir13sgd,jain2019making,zamani2023exact,liu2024revisiting}.
Notably, SGD in a stochastic non-realizable (either smooth or non-smooth) setup requires uniformly bounded noise assumptions, and generally cannot accommodate a constant step size independent of the optimization horizon.

\begin{table*}[hb!]
\centering
\small
\caption{
\textbf{State-of-the-art Loss Bounds for Fixed-Step-Size SGD.}
We consider stochastic convex optimization with an objective $\f(\w)\!\eqq\!\E_\xi f(\w;\xi)$, where $f(\cdot;\xi)$ is $\beta$-smooth almost surely, 
\linebreak
$\sigma^2 \geq \E\unorm[s]{\nabla f(\w;\xi) - \nabla \f(\w)}^2$, 
$\sigma_\star^2\eqq \E\unorm[s]{\nabla f(\teacher;\xi) - \nabla \f(\teacher)}^2$, and $G>0$ is such that
\linebreak
$\|\nabla f(\w;\xi)\|\leq G$ for any $\w$ and $\xi$.
\\
Dependence on constant numerical factors and the distance to an optimal solution is suppressed.
\label{table:wr_sgd}
} 
\vskip -0.1cm
\begin{tabular}{c c c c c} 
\Xhline{1.2pt}
\rule[-11pt]{0pt}{28pt}
\makecell{\textbf{Setting}} & 
\makecell{\textbf{Reference}}
&\makecell{\textbf{Bound} \\ \textbf{at Iteration $T$}}
&\makecell{\textbf{Last Iterate} \\ \textbf{Guarantee}} 
&\makecell{\textbf{Convergence}\\\textbf{for }$\boldsymbol{\eta=1/\beta}$}
\\
\Xhline{1.2pt}
\rule[-12pt]{0pt}{29pt}
\makecell{\vspace{-0pt} Stochastic }
& \makecell{
\hspace{-2em}
\texttt{(*)} \citet{shamir13sgd}
\hspace{-0.5em}
}
& $\displaystyle\frac1{\eta T} + \eta G^2 \log T$
& \cmark
& \xmark
\\
\cline{1-5}
\rule[-12pt]{0pt}{29pt}
\makecell{\vspace{1pt} Deterministic \\ Smooth ($\sigma=0$) }
& \makecell{\small \citet{nesterov1998introductory}}
& $\displaystyle\frac{1}{(2-\eta\beta)\eta T}$
& \cmark
& \cmark
\\ 
\cline{1-5}
\rule[-12pt]{0pt}{29pt}
\multirow{2}{*}{ \makecell{\vspace{0.3em}\\ Stochastic Smooth }}  
& \makecell{\small \citet{lan2012optimal}}
& $\displaystyle\frac1{\eta T} + \eta \sigma^2$
& \xmark
& \xmark
\\ 
\rule[-12pt]{0pt}{29pt}
& \makecell{\small \citet{liu2024revisiting}}
& $\displaystyle\frac1{\eta T} + \eta\sigma^2\log T$
& \cmark
& \xmark
\\
\cline{1-5}
\rule[-12pt]{0pt}{29pt}
\multirow{2}{*}{\makecell{\vspace{-20pt} 
    \\ Stochastic Smooth \\ Realizable ($\sigma_\star=0$)}} 
    & \makecell{\small \citet{srebro2010smoothness}
    }
    & $\displaystyle\frac{1}{(1-\eta\beta)\eta T}$
    & \xmark
    & \xmark
\\
\hline
\rule[-12pt]{0pt}{29pt}
\multirow{3}{*}{\makecell{\vspace{-3pt} \\\\ Stochastic 
\\
Regression \\ Realizable \\ ($\sigma_\star=0$)}} 
& \makecell{
\hspace{-1em}
\small \citet{bach2013non}
\hspace{-1em}
}
& $\displaystyle\frac{1}{\eta T}$
& \xmark
& \cmark
\\ 
\rule[-10pt]{0pt}{29pt}
& \makecell{\small \citet{varre2021last}}
& $\displaystyle\frac{1}{(1-2\eta\beta\log T)\eta T}$
& \cmark
& \xmark
\\ 
\rule[-12pt]{0pt}{29pt}
& \textbf{Ours}
& 
\hspace{-2em}
$\displaystyle\frac{1}{(2-\eta\beta)\eta T^{1-\eta\beta(1-\eta\beta/4)}}$
\hspace{-2em}
& \cmark
& \cmark
\\
\Xhline{1.2pt}
\end{tabular}
\vspace{.2em}
\caption*{\footnotesize \texttt{(*)} 
They consider bounded domains \citep{shamir13sgd};
 \citet{orabona2020last,liu2024revisiting} obtain similar bounds for the unconstrained case. 
 For non-fixed step sizes, \citet{jain2019making} obtain minimax optimal bounds without log factors.
}
%
\end{table*}

Our analysis for SGD \emph{without}-replacement is related to a long line of work primarily focused on the average iterate convergence rates 
\citep[e.g.,][]{recht2012toward, nagaraj2019sgd, safran2020good, rajput2020closing, mishchenko2020random, cha2023tighter,cai2023empirical}.
For the non-strongly convex case, near-optimal bounds (for the average iterate) have been established for the general smooth case \citep{nagaraj2019sgd, mishchenko2020random}. 
In a subsequent work, \citet{cai2023empirical} refined the dependence on problem parameters for the smooth realizable case (among others). 
Guarantees for the \emph{last} iterate have only been established recently by 
\citet{liu2024shuffling} and 
\citet{cai2025lastIterate}.
However, their bounds decay with the number of epochs rather than the number of iterations and apply only to non-constant step sizes, making them inapplicable to our setting.
Specifically, in a realizable $\beta$-smooth setup, after $J$ without-replacement SGD epochs over a finite sum of size $n$, \citet{mishchenko2020random,cai2023empirical} obtained an $O(\beta /J)$ bound for the average iterate with step size $\eta=1/(\beta n)$; and \citet{liu2024shuffling, cai2025lastIterate} derived similar bounds for the last iterate up to logarithmic factors.

Another line of work related to ours studies algorithmic stability
\citep{bousquet2002stability,shalev2010learnability} of gradient methods, which is the main technique we use in the proof of \cref{thm:wor_sgd_last_iterate_main}.
Our approach is similar in nature to that of \citet{nagaraj2019sgd,sherman2021optimal,koren2022benign} and primarily builds on \citet{sherman2021optimal}, who were the first to formally introduce the notion of without-replacement stability.  
For with-replacement SGD,
\cite{hardt2016train} discussed its algorithmic stability under smooth loss functions.  
Later, \cite{lei20stab}, improved this bound in the realizable loss case. 
The case we consider---\ie the stability of without-replacement SGD under smooth and realizable loss functions---is not covered in the existing literature.

\paragraph{With versus Without Replacement in Kaczmarz Methods.}
Our results in \cref{sec:rates} establish universal bounds for random orderings, both with and without replacement.
Both the with- and without-replacement variants converge linearly towards the minimum-norm solution $\teacher$ \citep{gower2015randomized,han2024reshuffling}, but as we explained in \cref{sec:discussion}, the rates can be arbitrarily slow.
\citet{recht2012noncommutativeAGM} formulated a noncommutative analog of the arithmetic-geometric mean inequality that, if true, could have shown that without-replacement orderings lead to faster loss convergence than with-replacement orderings in Kaczmarz methods, and consequently in continual linear regression.
Years later, \citet{lai2020recht} proved that this inequality does not hold in general \citep[see also][]{desa2020reshuffling}.
Moreover, as in other areas, empirical studies found that row shuffling followed by cyclic orderings performs as well as i.i.d.~orderings \citep{oswald2015convergence}.
This naturally connects to interesting observations and open questions regarding various forms of shuffled SGD \citep{bottou2009curiously,yun2021open}.
Our rates are similar for both with- and without-replacement orderings (up to small constants), meaning they do not indicate a clear advantage for either. However, we believe they are far from tight, leaving interesting open questions in this direction.


\paragraph{Connections to Normalized Least Mean Squares.}
The NLMS algorithm is a classical adaptive filtering method.
In its simplest version \citep{slock1993convergence}, the method perfectly fits a single---usually noisy---random sample at a time, using the same update rule as the Kaczmarz method (and thus, as our continual \cref{proc:regression_to_convergence} in a rank-$1$ case).
There also exists a more complex version of this method, which uses more samples per update \citep{sankaran2000convergence}.
Both papers give strong $\bigO(1/k)$ MSE rates in the noiseless setting (matching our realizable setting).
However, they assume a very limited data model, where the sampled vectors are either orthogonal or identical up-to-scaling.
Under such conditions, \citet{evron2022catastrophic} showed that there is no forgetting (of previously learned tasks), implying that the MSE decays as the number of tasks still unseen at time $k$.

\paragraph{Alternative Continual Schemes That Do Not Forget.}\!\!
Schemes such as Recursive Least Squares (RLS; see \citealt[Chapter~12.2]{chong2004introduction}) and its block variant BRMP \citep{zhuang2021blockwise} provide analytical alternatives to the block Kaczmarz method (\cref{scheme:kaczmarz}). 
These methods effectively avoid forgetting by maintaining an $\bigO(d^2)$ matrix. 
See also Proposition~5.5 in \citet{evron23continualClassification}.

Importantly, we study a continual scheme (\cref{proc:regression_to_convergence}) that closely characterizes common training practices. 
In particular, training with (S)GD to convergence coincides with the analytical updates of the (block) Kaczmarz method (\cref{scheme:kaczmarz}), making the latter illustrative of most gradient-based continual learning approaches. 
Understanding continual gradient-based algorithms in linear models is especially relevant given the linear dynamics of deep neural networks in the NTK regime \citep{jacot2018ntk}. 
These regimes are typically highly overparameterized, yet they remain covered by our analysis and by the dimensionality-independent rates we derive for the naturally forgetful but \emph{memoryless}  \cref{proc:regression_to_convergence}.

\newpage

\section{Auxiliary Proofs}
\label{app:auxiliary}

\begin{lemma}[Bounding Forgetting Using the Training Loss]
\label{lem:in_to_all_sample}
In a realizable setting (\cref{asm:realizability}),
the iterates of \cref{proc:regression_to_convergence} under a random task ordering $\tau$ (with or without replacement)
hold $\forall k\ge 1$,
\begin{align*}
\mathbb{E}_{\tau}
[F_{\tau}(k)]
&=
    \mathbb{E}_{\tau}
    \bigg[
    \frac{1}{2k}
    \sum_{t=1}^{k} 
    \norm{\X_{\tau(t)}\w_{k} - \y_{\tau(t)}}^2
    \bigg]
    %
\le 
\mathbb{E}_{\tau}\left\Vert \X_{\tau\left(k\right)}
\w_{k-1}-\mathbf{y}_{\tau\left(k\right)}
\right\Vert ^{2}
+
\frac{\norm{\teacher}^{2}R^{2}}{k}\,,
\end{align*}
where $R\triangleq \max_{m\in \cnt{T}} \norm{\X_m}$ is the ``radius'' of the data.
Notice that the dependence of $\w_{k-1}$ on $\tau_{1},\dots,\tau_{k-1}$ is implicit.
Particularly, in an ordering with replacement, we get,
\begin{align*}
\mathbb{E}_{\tau}
[F_{\tau}(k)]
&
\le
\mathbb{E}_{\tau}
\bigg[
\frac{1}{T}
\sum_{m=1}^{T}
\norm{\X_{m}\w_{k-1} - \y_{m}}^2
\bigg]
+
\frac{\norm{\teacher}^{2}R^{2}}{k}
=
2\mathbb{E}_{\tau}\left[\Loss\left(\w_{k-1}\right)\right]
+
\frac{\norm{\teacher}^{2}R^{2}}{k}
\,.
\end{align*}
\end{lemma}

\medskip

\begin{proof}
As discussed in \cref{sec:kaczmarz_reduction}, \cref{scheme:kaczmarz} governs the updates of the iterates \(\w_t \in \mathbb{R}^d\). 
Under \cref{asm:realizability}, we define the orthogonal projection as \(\mP_{\tau(t)} \eqq \I_d - \X_{\tau(t)}^+ \X_{\tau(t)}\), revealing a recursive form:
\begin{align}
    \w_t 
    & =
     \X_{\tau(t)}^{+}\y_{\tau(t)}+
     \left(\I_{\dim}-\X_{\tau(t)}^{+}\X_{\tau(t)}\right)\w_{t-1} \notag\\ 
    \explain{\text{\cref{asm:realizability}}}& =
     \X_{\tau(t)}^{+}\X_{\tau(t)}\teacher+
     \left(\I_{\dim}-\X_{\tau(t)}^{+}\X_{\tau(t)}\right)\w_{t-1} 
     =
     (\I_{\dim} - \mP_{\tau \prn{t}})\teacher+ \mP_{\tau \prn{t}}\w_{t-1} 
     \notag\\
     \w_t - \teacher &= \mP_{\tau \prn{t}}\prn{\w_{t-1}-\teacher} \label{eq:proj_rec} \\
     \w_{t} - \teacher&=\mP_{\tau\left(t\right)}\cdots\mP_{\tau\left(1\right)}\left(\w_{0}-\teacher\right) \label{eq:projected_errors}\,.
\end{align}
We show that,
\begin{align*}
\E_{\tau}\left[F_{\tau}\left(k\right)\right]	
&=
\frac{1}{2k}\sum_{t=1}^{k}\E_{\tau}\left\Vert \X_{\tau\left(t\right)}\w_{k}-\mathbf{y}_{\tau\left(t\right)}\right\Vert ^{2}=\frac{1}{2k}\sum_{t=1}^{k}\E_{\tau}\left\Vert \X_{\tau\left(t\right)}\left(\w_{k}-\teacher\right)\right\Vert ^{2}	
\\
&
    =\frac{1}{2k}\sum_{t=1}^{k}\E_{\tau}\left\Vert \X_{\tau\left(t\right)}\mP_{\tau\left(k\right)}\cdots\mP_{\tau\left(t+1\right)}\mP_{\tau\left(t\right)}\left(\w_{t-1}-\teacher\right)\right\Vert ^{2}
\\
&
    =\frac{1}{2k}\sum_{t=1}^{k}\E_{\tau}\left\Vert \X_{\tau\left(t\right)}\mP_{\tau\left(k\right)}\cdots\mP_{\tau\left(t+1\right)}\left(\I-\mP_{\tau\left(t\right)}-\I\right)\left(\w_{t-1}-\teacher\right)\right\Vert ^{2}
\\
\explain{\text{Jensen}}
& 
\le\frac{1}{k}
\sum_{t=1}^{k}
\bigg(~\E_{\tau}
\underbrace{\left\Vert \X_{\tau\left(t\right)}\mP_{\tau\left(k\right)}\cdots\mP_{\tau\left(t+1\right)}\left(\I-\mP_{\tau\left(t\right)}\right)\left(\w_{t-1}-\teacher\right)\right\Vert ^{2}}_{\le R^{2}\left\Vert \left(\I-\mP_{\tau\left(t\right)}\right)\left(\w_{t-1}-\teacher\right)\right\Vert ^{2}\text{, since projections contract}}
+
\
\\
&
\hspace{4.5em}
\E_{\tau}\left\Vert \X_{\tau\left(t\right)}\mP_{\tau\left(k\right)}\cdots\mP_{\tau\left(t+1\right)}\left(\w_{t-1}-\teacher\right)\right\Vert ^{2}
~
\bigg)
\\
&
	\le\frac{1}{k}
    \sum_{t=1}^{k}
    \bigg(~
    R^{2}\E_{\tau}\left\Vert \left(\I-\mP_{\tau\left(t\right)}\right)\left(\w_{t-1}-\teacher\right)\right\Vert ^{2}
+
\
\\
&
\hspace{4.5em}
    \E_{\tau}\left\Vert \X_{\tau\left(t\right)}\mP_{\tau\left(k\right)}\cdots\mP_{\tau\left(t+1\right)}\mP_{\tau\left(t-1\right)}\cdots\mP_{\tau\left(1\right)}\left(\w_{0}-\teacher\right)\right\Vert ^{2}
    ~\bigg)\,.
\end{align*}
For the first term, we employ the Pythagorean theorem for orthogonal projections to get a telescoping sum and show that
\begin{align*}
&
\frac{R^{2}}{k}\sum_{t=1}^{k}
\E_{\tau}\left\Vert 
\left(\I-\mP_{\tau\left(t\right)}\right)
\prn{\w_{t-1}-\teacher}\right\Vert ^{2}
\\
&
=
\frac{R^{2}}{k}\sum_{t=1}^{k}\left(\E_{\tau}\left\Vert \w_{t-1}-\teacher\right\Vert ^{2}-\E_{\tau}\left\Vert \mP_{\tau\left(t\right)}\left(\w_{t-1}-\teacher\right)\right\Vert ^{2}\right)
\\
&
=
\frac{R^{2}}{k}\sum_{t=1}^{k}\left(\E_{\tau}\left\Vert \w_{t-1}-\teacher\right\Vert ^{2}-
\E_{\tau}\left\Vert 
\w_{t}-\teacher\right\Vert ^{2}\right)
\\
&
=
\frac{R^{2}}{k}
\biggprn{\underbrace{\E_{\tau}\left\Vert \w_{0}-\teacher\right\Vert ^{2}}_{=\left\Vert \teacher\right\Vert ^{2}}-\underbrace{\E_{\tau}\left\Vert \w_{k}-\teacher\right\Vert ^{2}}_{\ge0}
}
\le 
\frac{\norm{\teacher}^{2} R^{2}}{k}\,.
\end{align*}

For the second term, we use the exchangeability of $\tau$ which applies with or without replacement,
\begin{align*}
&\E_{\tau}\left\Vert \X_{\tau\left(t\right)}\mP_{\tau\left(k\right)}\cdots\mP_{\tau\left(t+1\right)}\mP_{\tau\left(t-1\right)}\cdots\mP_{\tau\left(1\right)}\left(\w_{0}-\teacher\right)\right\Vert ^{2}
\\
&
=
\E_{\tau}\left\Vert \X_{\tau\left(k\right)}\mP_{\tau\left(k-1\right)}\cdots\mP_{\tau\left(1\right)}\left(\w_{0}-\teacher\right)\right\Vert ^{2}
=
\E_{\tau}\left\Vert \X_{\tau\left(k\right)}\left(\w_{k-1}-\teacher\right)\right\Vert ^{2}
\,.
\end{align*}

Combining the two, we get
\begin{align*}\E_{\tau}\left[F_{\tau}\left(k\right)\right]
&
	\le\E_{\tau}\left\Vert \X_{\tau\left(k\right)}\w_{k-1}-\mathbf{y}_{\tau\left(k\right)}\right\Vert ^{2}
    +\frac{\left\Vert \teacher\right\Vert ^{2}R^{2}}{k}
    \,,
\end{align*}
which completes the first part of the proof.

For the second part, simply notice that in an i.i.d.~setting,
the index $\tau(k)\sim \Unif\prn{\cnt{T}}$ is independent of earlier indices (which yielded $\w_{k-1}$), and thus
$$\E_{\tau}\left\Vert \X_{\tau\left(k\right)}\w_{k-1}-\mathbf{y}_{\tau\left(k\right)}\right\Vert ^{2}
=
\E_{\tau}
\bigg[
\frac{1}{T}
\sum_{m=1}^{T}
\norm{\X_{m}\w_{k-1} - \y_{m}}^2
\bigg]\,.
$$
\end{proof}

\bigskip

\begin{proposition}[Bounding The Training Loss Using Forgetting in Without-Replacement Orderings]
\label{prop:loss_to_forgetting}
Under a random ordering $\tau$ without replacement, the iterates of \cref{proc:regression_to_convergence} (continual regression)  satisfy $\forall k \in \sqprn{T}$:
\begin{align*}
\mathbb{E}_{\tau}\left[\mathcal{L}\left(\w_{k}\right)\right]=\frac{k}{T}\mathbb{E}_{\tau}\left[F_{\tau}\left(k\right)\right]+\frac{T-k}{2T}\mathbb{E}_{\tau}\left\Vert \mathbf{X}_{\tau\left(k+1\right)}\w_{k}-\mathbf{y}_{\tau\left(k+1\right)}\right\Vert ^{2} \, .
\end{align*}
Similarly, the iterates of \cref{scheme:pocs} (POCS) satisfy:
\begin{align*}
\mathbb{E}_{\tau}\left[\mathcal{L}\left(\w_{k}\right)\right]=\frac{k}{T}\mathbb{E}_{\tau}\left[F_{\tau}\left(k\right)\right]+\frac{T-k}{2T}\mathbb{E}_{\tau}\left\Vert \w_{k}-\proj_{\tau\left(k+1\right)}\left(\w_{k}\right)\right\Vert ^{2} \, ,
\end{align*}
where in such a POCS setting, the loss and forgetting are defined as:
\begin{align*}
    \mathcal{L}\left(\w_{k}\right)&=\frac{1}{2T}\sum_{m=1}^{T}\left\Vert \w_{k}-\proj_{m}\left(\w_{k}\right)\right\Vert ^{2}, \quad 
    F_{\tau}\left(k\right)=\frac{1}{2k}\sum_{t=1}^{k}\left\Vert \w_{k}-\proj_{\tau\left(t\right)}\left(\w_{k}\right)\right\Vert ^{2} \, .
\end{align*}
\end{proposition}

\medskip

\begin{proof} We first prove the claim in the continual regression setting. If $k=T$ then $\mathbb{E}_{\tau}\left[\mathcal{L}\left(\w_{k}\right)\right]=\mathbb{E}_{\tau}\left[F_{\tau}\left(k\right)\right]$, 
and the claim follows. For $k<T$, we have:
\begin{align*}
\mathbb{E}_{\tau}\left[\mathcal{L}\left(\w_{k}\right)\right] & =\frac{1}{2T}\sum_{m=1}^{T}\mathbb{E}_{\tau}\left\Vert \mathbf{X}_{m}\w_{k}-\mathbf{y}_{m}\right\Vert ^{2} \\
\explain{\text{without replacement}}&=
\frac{1}{2T}\sum_{t=1}^{T}\mathbb{E}_{\tau}\left\Vert \mathbf{X}_{\tau\left(t\right)}\w_{k}-\mathbf{y}_{\tau\left(t\right)}\right\Vert ^{2}
\\
 & =
 \frac{1}{2T}\sum_{t=1}^{k}\mathbb{E}_{\tau}\left\Vert \mathbf{X}_{\tau\left(t\right)}\w_{k}-\mathbf{y}_{\tau\left(t\right)}\right\Vert ^{2}
 +
 \frac{1}{2T}\sum_{t=k+1}^{T}\mathbb{E}_{\tau}\left\Vert \mathbf{X}_{\tau\left(t\right)}\w_{k}-\mathbf{y}_{\tau\left(t\right)}\right\Vert ^{2}
 \\
 & =\frac{k}{T}\mathbb{E}_{\tau}\left[F_{\tau}\left(k\right)\right]+\frac{1}{2T}\sum_{t=k+1}^{T}\mathbb{E}_{\tau}\left\Vert \mathbf{X}_{\tau\left(t\right)}\w_{k}-\mathbf{y}_{\tau\left(t\right)}\right\Vert ^{2}
 \\
\explain{\text{exchangeability}} & =\frac{k}{T}\mathbb{E}_{\tau}\left[F_{\tau}\left(k\right)\right]+\frac{T-k}{2T}\mathbb{E}_{\tau}\left\Vert \mathbf{X}_{\tau\left(k+1\right)}\w_{k}-\mathbf{y}_{\tau\left(k+1\right)}\right\Vert ^{2} \,. 
\end{align*}
For the POCS case, simply replace $\left\Vert \mathbf{X}_{m}\w_{k}-\mathbf{y}_{m}\right\Vert ^2$ with $\left\Vert \w_{k}-\proj_{m}\left(\w_{k}\right)\right\Vert ^{2}$.
\end{proof}

\bigskip
\bigskip

\begin{recall}[\cref{lem:cl_gd_equiv}]
Consider any realizable task collection such that $\X_m\teacher=\y_m, \forall m\in\cnt{T}$.
Define $f_{m}(\w) = 
\frac{1}{2}\norm{\X_{m}^+\X_{m}\prn{\w - \teacher}}^{2}$.
Then, $\forall m\in \cnt{T}, \w\in\R^{d}$
\begin{enumerate}[label=(\roman*), leftmargin=*,itemsep=0pt]
\item \textbf{Upper bound:} 
$\mathcal{L}_{m}(\w) \le R^2 f_{m} (\w)
\triangleq
\max_{m'\in \cnt{T}} \norm{\X_{m'}}^2 f_{m}$\,.

\item \textbf{Gradient:}
\hspace{1.5em}
$\nabla_{\w} f_m(\w) = \X_{m}^+\X_{m}\prn{\w - \teacher}
 = \X_{m}^+\X_{m}\w - \X_{m}^{+}\y_m$\,.

\item \textbf{Convexity and Smoothness:}
$f_m$ is convex and $1$-smooth.

\end{enumerate}
\end{recall}
\medskip
\begin{proof}
First, we use the realizability and simple norm inequalities to obtain,
    $$
    \mathcal{L}_{m}(\w)
    =
    \tfrac{1}{2}
    \norm{\X_m\w-\y_m}^2
    =
    \tfrac{1}{2}
    \norm{\X_m(\w-\teacher)}^2
    \le 
    \tfrac{\norm{\X_m}^2}{2}
    \norm{\X_{m}^{+}\X_{m}(\w-\teacher)}^2
    \le
    R^2 f(\w)\,.
    $$
Since $\X_m^+\X_m$ is an orthogonal projection operator---and thus symmetric and idempotent---we get,
    \begin{align*}
         \nabla_{\w} f_m(\w) 
         = \br{\X_{m}^{+}\X_{m}}^{\top}
         {\X_{m}^{+}\X_{m}} (\w -\teacher)
         = \X_{m}^{+} \X_{m} (\w -\teacher)
         = \X_{m}^{+} \X_{m} \w - \X_{m}^{+}\y_m\,.
    \end{align*}
    Then, the above and the fact that projection operators are non-expansive imply that $\forall \w,\vz\in\R^{d}$,
    $$
    \norm{
    \nabla_{\w} f_m(\w) - \nabla_{\vz} f_m(\vz)
    }
    =
    \norm{
    \X_{m}^{+} \X_{m} (\w -\teacher - \vz+\teacher)
    }
    =
    \norm{
    \X_{m}^{+} \X_{m} (\w - \vz)
    }
    \le
    \norm{\w - \vz
    }\,.
    $$
    Finally, the convexity of $f_m$ is immediate since $\nabla_{\w}^{2} f_{m} (\w)
    = \X_{m}^{+} \X_{m} \succeq \0$.
\end{proof}
\newpage
\section{Proofs for \secref{sec:parameter-dependent}:
A Parameter-Dependent $\bigO(1/k)$ Rate}
\label{app:parameter-dependent}

\begin{recall}[\cref{thm:random_convergence_rate}]
Under a random ordering with replacement 
over $T$ jointly realizable tasks, 
the expected loss and forgetting of {Schemes~\ref{proc:regression_to_convergence},~\ref{scheme:kaczmarz}} after \( k \geq 3\) iterations are upper bounded as,
\begin{align*}
\mathbb{E}_{\tau}\left[\Loss\left(\w_{k}\right)\right]
&
=
\mathbb{E}_{\tau}
\bigg[
\frac{1}{2T} 
\sum_{m=1}^{T} \norm{\X_m \w_k - \y_m}^2
\bigg]
\le
\frac{\min\prn{\sqrt{d - \rankavg},\sqrt{T\rankavg}}}{
2e(k - 1)}
\tnorm{\teacher}^2 R^2
\\
\mathbb{E}_{\tau}
\left[
F_{\tau}(k)
\right]
&
=
\mathbb{E}_{\tau}
\bigg[
\frac{1}{2k}
\sum_{t=1}^{k} 
\norm{\X_{\tau(t)}\w_{t} - \y_{\tau(t)}}^2
\bigg]
%
%
\leq 
\frac{3\min\prn{\sqrt{d - \rankavg},\sqrt{T\rankavg}}}{
2\prn{k - 2}}
\tnorm{\teacher}^2 R^2
\,,
\end{align*}
where $\rankavg\triangleq\frac{1}{T}\sum_{m \in [T]} \rank(\X_m)$.
(Recall that ${R\triangleq\max_{m \in [T]} \norm{\X_m}}$.)
\end{recall}

Here, we prove the main result, followed by auxiliary corollaries and lemmas in \cref{sec:app_aux_dim_dep}.
\label{app:main-proofs}
\paragraph{Proof Idea.}
We rewrite the Kaczmarz update (\cref{scheme:kaczmarz}) in a recursive form of the differences, \ie $\w_t - \teacher = \mP_{\tau(t)}\prn{\w_{t-1} - \teacher}$, with a suitable projection matrix $\mP_{\tau(t)}$.
We define the linear map $Q\sqprn{\A} = \frac{1}{T}\sum_{m=1}^{T} \mP_m \A \mP_m$ to capture the evolution of the difference's second moments, enabling sharp analysis of the expected loss in terms of $Q$. 
Using properties of $Q$, norm inequalities, and the spectral mapping theorem, we establish a fast $\bigO \prn{1/k}$ rate with explicit dependence on $T$, $d$, and $\rankavg$.
\begin{proof}
We analyze the randomized block Kaczmarz algorithm for solving the linear system \(\X \w = \y\), where the matrix and vector are partitioned into blocks as follows:
\begin{equation*}
	\X =
	\begin{pmatrix}
		\X_{1}     \\
		\vdots \\
		\X_{T}
	\end{pmatrix}, \quad 
	\y =
	\begin{pmatrix}
		\y_{1}     \\
		\vdots \\
		\y_{T}
	\end{pmatrix}\,.
\end{equation*}
By defining \(\vz_t = \w_t - \teacher\) and exploiting the recursive form of \cref{eq:proj_rec} from the proof of \cref{lem:in_to_all_sample}, we obtain \(\vz_t = \mP_{\tau(t)} \vz_{t-1}\).  Note that $\vz_0 = \0_d-\teacher=-\teacher$. \newline
Now, define the linear map $Q : \reals^{d \times d}\to \reals^{d \times d}$
as
\begin{align}
\label{eq:q-map}
    Q\sqprn{\A}= \expectation_{m \sim \text{Unif}\prn{\sqprn{T}}}\sqprn{\mP_{m} \A \mP_{m}}
    = \frac{1}{T}\sum_{m=1}^{T}\mP_{m}\A \mP_{m}.
\end{align}
This map plays a central role in our analysis and has been studied in similar forms in prior work \citep{guo2022rates}.
Note that $\mP_{m}$ is an orthogonal projection, \ie symmetric and idempotent. Thus,  
\begin{align*}
	\expectation_{\tau}\!\sqprn{\vz_{t+1} \vz_{t+1}^\top} 
    & \!= 
    \!\expectation_{m, \tau}\!\sqprn{\mP_{m} \vz_{t} \vz_{t}^\top \mP_{m}^\top}
    \!=
    \!\expectation_{m, \tau}\!\sqprn{\mP_{m} \vz_{t} \vz_{t}^\top \mP_{m}}
    \!=
    \expectation_{m}\sqprn{\mP_{m} \expectation_{\tau}\sqprn{\vz_{t} \vz_{t}^\top} \mP_{m}}
    \!
    = Q\sqprn{\expectation_{\tau}\sqprn{\vz_{t} \vz_{t}^\top}}\!.
\end{align*}
It follows that $$\expectation_{\tau}\sqprn{\vz_{t} \vz_{t}^\top}
= Q^{t}\sqprn{\expectation_{\tau}\sqprn{\vz_{0} \vz_{0}^\top}}
=
Q^{t}\sqprn{\vz_{0} \vz_{0}^\top}
= 
Q^{t}\sqprn{\prn{\w_0-\teacher} \prn{\w_0-\teacher}^\top}
= 
Q^{t}\sqprn{\teacher \teachertop}\,,$$ 
where \( Q^{t} \) 
denotes \( t \) applications of  \( Q \).
The map \( Q \) captures the evolution of the error's second-moment under Kaczmarz updates, offering a tractable approach to analyzing the algorithm's convergence.

The expected loss at step $t$ is given by
\begin{align*} \label{eq:loss}
	&\mathbb{E}_{\tau}\sqprn{\Loss\prn{\w_{t}}} =\expectation_{\tau}\sqprn{\frac{1}{2T} \sum_{i=1}^{T} \norm{\X_{i} \w_{t} - \y_{i}}^2}
	         = \expectation_{\tau}
             \sqprn{\frac{1}{2T} \sum_{i=1}^{T} \norm{\X_{i} \prn{\w_{t} - \teacher}}^2}
             \\ &
             = \expectation_{\tau}
             \sqprn{\frac{1}{2T} \sum_{i=1}^{T} \norm{\X_{i} \vz_{t}}^2}
             \notag                                            
             = \expectation_{\tau}\sqprn{\frac{1}{2T} \norm{\X \vz_{t}}^2}
             = \expectation_{\tau}\sqprn{\frac{1}{2T} \vz_{t}^\top \X^\top \X \vz_{t}}
             \notag
             \\
             &
             = \expectation_{\tau}\sqprn{\tr\prn{\frac{1}{2T} \X^\top \X \vz_{t} \vz_{t}^\top}}
             = 
             \tr\prn{\frac{1}{2T} \X^\top \X \expectation_{\tau}\sqprn{\vz_{t} \vz_{t}^\top}}
             = 
             \tr\prn{\frac{1}{2T} \X^\top \X Q^t \sqprn{\teacher \teachertop}}. 
\end{align*}
We are now ready to derive the final bound. 
From \cref{lem:A Upper Bound}, we have  
\[
	\frac{1}{R^2 T}\X^{\top}\X \preccurlyeq \X^{+} \X - Q\sqprn{\X^{+} \X}.
\]  
Additionally, by Corollary~\ref{rem:Qk-is-positive}, \( Q^{k}\sqprn{\teacher \teachertop} \) is symmetric and positive semidefinite (PSD). We also note that \( \frac{1}{T}\X^{\top}\X \) is symmetric PSD.
The key insight from Lemma~\ref{lem:A Upper Bound}, combined with the trace product inequality (Lemma~\ref{lem:Trace-inequality}), is that it allows the expected loss to be expressed using a polynomial in \( Q \). This reformulation simplifies the convergence analysis by reducing it to examining the spectral properties of \( Q \).
Invoking the trace product inequality, we obtain:
\begin{align*}
\mathbb{E}_{\tau}\sqprn{\Loss\prn{\w_{k}}} 
& 
=
\tr\prn{\frac{1}{2T} \X^\top \X Q^t \sqprn{\teacher \teachertop}}
\leq
\frac{R^2}{2} \tr\prn{\prn{\X^{+}\X - Q\sqprn{\X^{+}\X}} Q^{k}\sqprn{\teacher \teachertop}}\notag 
\\
\explain{\text{\cref{lem:Q trace}}}
    & =
    \frac{R^2}{2} \tr\prn{Q^{k}\sqprn{\X^{+} \X - Q\sqprn{\X^{+} \X}}\teacher\teachertop} 
    =
    \frac{R^2}{2} \teachertop Q^{k}\sqprn{\X^{+} \X - Q\sqprn{\X^{+} \X}}\teacher
    \\
    & \leq
    \frac{\norm{\teacher}^{2} R^{2}}{2} \norm{Q^{k}\sqprn{\X^{+} \X - Q\sqprn{\X^{+} \X}}}_2 =\frac{\norm{\teacher}^{2} R^{2}}{2} \norm{\bigprn{Q^{k}\prn{I-Q}}\sqprn{\X^{+} \X}}_{2}\\
              & = \frac{\norm{\teacher}^{2} R^{2}}{2} 
              \norm{\bigprn{Q^{k-1}\prn{I-Q}}Q\sqprn{\X^{+} \X}}_{2}
        \\ 
        &
        \leq \frac{\norm{\teacher}^{2} R^{2}}{2} 
              \norm{\bigprn{Q^{k-1}\prn{I-Q}}Q\sqprn{\X^{+} \X}}_{F} 
              \\
      \explain{\text{operator norm}}
              & 
              \leq \frac{\norm{\teacher}^{2} R^{2}}{2} \norm{Q^{k-1}\prn{I-Q}} \cdot \norm{Q\sqprn{\X^{+} \X}}_{F} \\
    \explain{\text{Lemmas}\\\text{\ref{cor:polynomial_Q_spectrum},\ \ref{cor:dimensionality_bound}}}
    &
    \leq \frac{\norm{\teacher}^{2} R^{2}}{2e\prn{k-1}} \min\prn{\sqrt{T\rankavg}, \sqrt{d-\rankavg}}
    \,.
\end{align*}
To clarify, the operator norm of a linear map $H$ is defined as $\norm{H}=\sup_{\A\in\reals^{d\times d}, \norm{\A}_F=1} \norm{H\sqprn{\A}}_F$. The reason for switching from the spectral norm to the Frobenius norm is to enable the use of the spectral mapping theorem to bound the operator norm of $Q^{k-1}\prn{I-Q}$,  applicable only for inner-product-based norms.
We complete the proof by bounding the forgetting using the training loss (\cref{lem:in_to_all_sample}).
That is,
\begin{align*}
\mathbb{E}_{\tau}
[F_{\tau}(k)]
&
=
\mathbb{E}_{\tau}
\bigg[
\frac{1}{2k}
\sum_{t=1}^{k} 
\norm{\X_{\tau(t)}\w_{t} - \y_{\tau(t)}}^2
\bigg]
\le
2\mathbb{E}_{\tau}\sqprn{\Loss\prn{\w_{k-1}}}
+
\frac{\norm{\teacher}^{2}R^{2}}{k}
\\
& \leq
\frac{3\norm{\teacher}^{2} R^{2}}{2\prn{k-2}} \min\prn{\sqrt{T\rankavg}, \sqrt{d-\rankavg}}
\,.
\end{align*}
\end{proof}

\newpage

\subsection{Key Properties and Auxiliary Lemmas}
\label{sec:app_aux_dim_dep}

\begin{definition}[Positive Map]
	A positive map $H:\mathbb{R}^{d\times d}\to\mathbb{R}^{d\times d}$ is a linear
	map that maps PSD matrices to PSD matrices. Formally, if $\mathbf{0}\preccurlyeq
	\A\in\mathbb{R}^{d\times d}$, then $\mathbf{0}\preccurlyeq H\sqprn{\A}$.
\end{definition}

\begin{definition}[Symmetric Map]
	A symmetric map $H:\mathbb{R}^{d\times d}\to\mathbb{R}^{d\times d}$ is a linear
	map that maps symmetric matrices to symmetric matrices. Formally, if $\A =
	\A^{\top}\in\mathbb{R}^{d\times d}$, then $H\sqprn{\A}= H\sqprn{\A}^{\top}$.
\end{definition}

\begin{corollary}
	\label{cor:Q-is-positive} $Q$, defined in \cref{eq:q-map}, is a positive map.
\end{corollary}

\begin{proof}
	Let $\mathbf{0}\preccurlyeq\A\in\mathbb{R}^{d\times d}$. Then, for all $i \in \sqprn{T}$, $\0 \preccurlyeq \mP_{i}\A \mP_{i}$. Meaning $Q\sqprn{\A}$ is PSD as a convex combination of PSD matrices.
\end{proof}

\begin{corollary}
	\label{cor:Q-is-symmetric} $Q$ is a symmetric map. Moreover, for all $\A \in\mathbb{R}^{d\times d}$, it satisfies $Q\sqprn{\A}^\top = Q\sqprn{\A^\top}$.
\end{corollary}

\begin{proof}
	Let $\A\in\mathbb{R}^{d\times d}$. Then,
	\begin{align*}
		Q\sqprn{\A}^{\top}           = \frac{1}{T}\sum_{i=1}^{T}\prn{\mP_{i} \A \mP_{i}}^{\top}
		                             = \frac{1}{T}\sum_{i=1}^{T}\mP_{i}^{\top}\A^{\top}\mP_{i}^{\top}= \frac{1}{T}\sum_{i=1}^{T}\mP_{i}\A^{\top}\mP_{i}= Q\sqprn{\A^\top}.
	\end{align*}
\end{proof}

\begin{corollary}
	\label{rem:Qk-is-positive} For $n\in\mathbb{N}^{+}$, the iterated application of the map $Q$, denoted $Q^n$, is a positive symmetric map.
\end{corollary}

\begin{proof}
	For $n=1$, given by \cref{cor:Q-is-positive,cor:Q-is-symmetric}.
    For $n > 1$, this follows trivially by induction.
\end{proof}

\begin{lemma}[Trace Product Inequality]
	\label{lem:Trace-inequality}
Let $\A, \B, \C\in\mathbb{R}^{d\times d}$ be symmetric PSD matrices such that $\A\preccurlyeq\B$. 
Then, $\tr\prn{\A\C}\leq\tr\prn{\B\C}$.
\end{lemma}

\begin{proof}
Since $\mathbf{0}\preccurlyeq\C = \C^{\top}$, it has a square symmetric PSD root $\C^{1/2}$.
Given that $\A,\B$ are symmetric and $\A\preccurlyeq \B$, 
it follows that $\C^{1/2}\A \C^{1/2}\preccurlyeq \C^{1/2}\B \C^{1/2}$ 
\citep[from][Theorem~7.7.2.a]{Horn_Johnson_2012}.
Applying the cyclic property of the trace and using the fact that for symmetric matrices ordered in the Löwner sense, their traces are also ordered \cite[Corollary~7.7.4.d]{Horn_Johnson_2012}, we obtain
	\begin{align*}
		\tr\prn{\A \C} =\tr\prn{\A \C^{1/2} \C^{1/2}}= \tr\prn{\C^{1/2} \A \C^{1/2}} \leq \tr\prn{\C^{1/2} \B \C^{1/2}} = \tr\prn{\B \C}.
	\end{align*}
\end{proof}

\newpage

\begin{lemma}
	\label{lem:A Upper Bound}
    Let $R=\max_{i \in \sqprn{T}}{\norm{\X_i}}$. Then, $\frac{1}{R^2 T}\X^{\top}\X \preccurlyeq \X^{+} \X - Q\sqprn{\X^{+} \X}$
\end{lemma}
\begin{proof}
    We perform SVD on each \(\X_{i} = \U_{i} \mSigma_{i} \V_{i}^{\top}\). Then,
    \begin{align*}
        \frac{1}{R^2 T} \X^{\top} \X & = \frac{1}{R^2 T} \sum_{i=1}^{T} \X_{i}^{\top} \X_{i} = \frac{1}{R^2 T} \sum_{i=1}^{T} \V_{i} \mSigma_{i}^{2} \V_{i}^{\top}
    \end{align*}
    On the other hand:
    \begin{align*}
    &
        \X^{+} \X - Q\sqprn{\X^{+} \X}  = \X^{+} \X - \frac{1}{T} \sum_{i=1}^{T} \prn{\I - \X_{i}^{+} \X_{i}} \X^{+} \X \prn{\I - \X_{i}^{+} \X_{i}} \notag \\
                                       & = \X^{+} \X - \frac{1}{T} \sum_{i=1}^{T} \X^{+} \X - \X_{i}^{+} \X_{i} \X^{+} \X - \X^{+} \X \X_{i}^{+} \X_{i} 
        + \X_{i}^{+} \X_{i} \X^{+} \X \X_{i}^{+} \X_{i} \notag 
        \\
        \explain{\text{Im}\prn{\X_{i}^{+} \X_{i}} \\ \subseteq \text{Im}\prn{\X^{+} \X}} & = -\frac{1}{T} \sum_{i=1}^{T} -\X_{i}^{+} \X_{i} - \X_{i}^{+} \X_{i} + \X_{i}^{+} \X_{i} = \frac{1}{T} \sum_{i=1}^{T} \X_{i}^{+} \X_{i} 
        = \frac{1}{T} \sum_{i=1}^{T} \V_{i} \mSigma_{i}^{+} \mSigma_{i} \V_{i}^{\top}.
    \end{align*}
    Now consider the difference:
    \begin{align*}
        \prn{\X^{+} \X - Q\sqprn{\X^{+} \X}} - \frac{1}{R^2T} \X^{\top} \X & = \frac{1}{T} \sum_{i=1}^{T} \V_{i} \prn{\mSigma_{i}^{+} \mSigma_{i} - \frac{1}{R^2}\mSigma_{i}^{2}} \V_{i}^{\top}.
    \end{align*}
    We know that \(\frac{1}{R}\prn{\mSigma_{i}}_{j,j} \in [0, 1]\).
    We analyze two cases for each diagonal entry:
    \begin{itemize}
    \item If \(\prn{\mSigma_{i}}_{j,j} = 0\), then \(\prn{\mSigma_{i}^{+} \mSigma_{i} - \frac{1}{R^2}\mSigma_{i}^{2}}_{j,j} = 0\).
    
    \item
    Otherwise, \(\prn{\mSigma_{i}^{+} \mSigma_{i}}_{j,j} = 1\), and \(\frac{1}{R^2}\prn{\mSigma_{i}^{2}}_{j,j} \leq 1\), which gives \(\prn{\mSigma_{i}^{+} \mSigma_{i} - \frac{1}{R^2}\mSigma_{i}^{2}}_{j,j} \geq 0\).
    \end{itemize}
    Thus,
    \[
        \mathbf{0} \preccurlyeq \V_{i} \prn{\mSigma_{i}^{+} \mSigma_{i} - \frac{1}{R^2} \mSigma_{i}^{2}} \V_{i}^{\top}\,.
    \]
    Averaging over all \(i\), we get:
    \begin{align*}
        \mathbf{0} = 
        \frac{1}{T} \sum_{i=1}^{T} \mathbf{0} 
        & 
        \preccurlyeq \frac{1}{T} \sum_{i=1}^{T} \V_{i} \prn{\mSigma_{i}^{+} \mSigma_{i} - \frac{1}{R^2} \mSigma_{i}^{2}} \V_{i}^{\top} 
        =
        \prn{\X^{+} \X - Q\sqprn{\X^{+} \X}} - \frac{1}{R^2 T} \X^{\top} \X 
        \notag 
        \\
        \frac{1}{R^2 T} \X^{\top} \X              & \preccurlyeq \X^{+} \X - Q\sqprn{\X^{+} \X}
        \,.
    \end{align*}
\end{proof}

\newpage

\begin{lemma}
	\label{lem:Q trace} Let $\A, \B \in \mathbb{R}^{d \times d}$ and $n \in \mathbb{N}
	^{+}$. Then, $\tr\prn{\A Q^{n}\sqprn{\B}}= \tr\prn{Q^{n}\sqprn{\A} \B}$.
\end{lemma}

\begin{proof}
	From the definition of $Q$ (\cref{eq:q-map}),
	\begin{align*}
		\tr\prn{\A Q^{n}\sqprn{\B}}                  & = \tr\prn{\A \frac{1}{T^{n}} \sum_{j_{1}, \ldots, j_{n}=1}^{T} \mP_{j_{1}} \cdots \mP_{j_{n}} \B \mP_{j_{n}} \cdots \mP_{j_{1}}} \notag       
        \\
		\explain{\text{linearity}}       & = \frac{1}{T^{n}}\sum_{j_{1}, \ldots, j_{n}=1}^{T}\tr\prn{\A \mP_{j_{1}} \cdots \mP_{j_{n}} \B \mP_{j_{n}} \cdots \mP_{j_{1}}} 
        \notag         
        \\
		\explain{\text{cyclic property}} & = \frac{1}{T^{n}}\sum_{j_{1}, \ldots, j_{n}=1}^{T}\tr\prn{\mP_{j_{n}} \cdots \mP_{j_{1}} \A \mP_{j_{1}} \cdots \mP_{j_{n}} \B} 
        \notag         
        \\
		\explain{\text{linearity}}      
        & = \tr\prn{\prn{\frac{1}{T^{n}} \sum_{j_{1} \ldots, j_{n}=1}^{T} \mP_{j_{n}} \cdots \mP_{j_{1}} \A \mP_{j_{1}} \cdots \mP_{j_{n}}} \B} \notag \\
         & = \tr\prn{Q^{n}\sqprn{\A} \B}\,.
	\end{align*}
\end{proof}

\begin{proposition}
	\label{cor:Q self adjoint} $Q$ is self adjoint.
\end{proposition}

\begin{proof}
	Let $\A, \B \in \mathbb{R}^{d \times d}$. Then,
	\begin{align*}
		\langle Q\sqprn{\A}, \B \rangle                        & = \tr\prn{Q\sqprn{\A}^{\top} \B}= \tr\prn{\B^{\top} Q\sqprn{\A}} \notag                          \\
		\explain{\text{Lemma~\ref{lem:Q trace}}}            & = \tr\prn{Q\sqprn{\B^{\top}} \A} \notag                                                          \\
		\explain{\text{\cref{cor:Q-is-symmetric}}} 
        & = 
        \tr\prn{Q\sqprn{\B}^{\top} \A}= \tr\prn{\A^{\top} Q\sqprn{\B}}= \langle \A, Q\sqprn{\B}\rangle\,.
	\end{align*}
\end{proof}

\begin{proposition}
\label{cor:Q spectrum} 
The spectrum of $Q$ is contained in the interval $[0, 1]$.
\end{proposition}

\begin{proof}
Let $\A \in \mathbb{R}^{d \times d}$. Then, by definition,
\begin{align*}
    \langle Q\sqprn{\A}, \A \rangle 
    & = \tr\prn{Q\sqprn{\A}^{\top} \A} 
    = \frac{1}{T}\sum_{i=1}^{T}\tr\prn{\mP_{i} \A^{\top} \mP_{i} \A} 
    \\
    \explain{\text{idempotence,}\\ 
    \text{cyclic property}}
    & 
    = \frac{1}{T}\sum_{i=1}^{T}\tr\prn{\mP_{i} \A^{\top} \mP_{i} \mP_{i} \A \mP_{i}} 
    = \frac{1}{T}\sum_{i=1}^{T}\norm{\mP_{i} \A \mP_{i}}_{F}^{2} \geq 0\,.
\end{align*}
Since each $\mP_i$ is an orthogonal projection, its spectral norm satisfies $\norm{\mP_i}_2 = 1$. 
Applying the operator inequality $\norm{\X \Y}_F \leq \norm{\X}_2 \norm{\Y}_F$ twice, we obtain
    \[
    \frac{1}{T}\sum_{i=1}^{T}\norm{\mP_{i} \A \mP_i}_{F}^{2} 
    \le \norm{\mP_i}_{2}^{4} \norm{\A}_{F}^{2} 
    = 
    \norm{\A}_{F}^{2}\,.
\]
Thus, for any $\A\in \mathbb{R}^{d\times d}$,
\[
    0 \le \langle Q\sqprn{\A}, \A \rangle \le \norm{\A}_{F}^{2}\,.
\]
From the Rayleigh quotient characterization of eigenvalues, this implies that every eigenvalue $\lambda$ of $Q$ satisfies $0 \le \lambda \le 1$, $\ie$ $\sigma(Q) \subset [0, 1]\,$.
\end{proof}

\bigskip

\begin{lemma}
	\label{cor:polynomial_Q_spectrum} $\norm{Q^{n}\prn{I - Q}}\leq \frac{1}{en}, \quad
	\text{for }n \in \mathbb{N}^{+}.$
\end{lemma}

\begin{proof}
By~\cref{cor:Q self adjoint}, $Q$ is self adjoint. Thus, we can apply the spectral mapping theorem to the polynomial $x \mapsto x^{n}\prn{1 - x}$. The eigenvalues of $Q^{n}\prn{I - Q}$ are of the form $\lambda^{n}\prn{1 - \lambda}$, where $\lambda$ is an eigenvalue of $Q$. From~\cref{cor:Q spectrum}, we know that $\lambda \in \left[0, 1\right]$.
Using an algebraic property of $\lambda^{n}\prn{1-\lambda}$ for $\lambda \in \sqprn{0,1}$, we conclude that $\lambda^{n}\prn{1-\lambda} \in \left[0, \frac{1}{en}\right]$. \newline
Therefore, $\norm{Q^{n}\prn{I - Q}} \leq \frac{1}{en}$.
\end{proof}

\bigskip
\begin{lemma}
    \label{cor:dimensionality_bound}
    $\norm{Q\sqprn{\X^{+}\X}}_{F} \leq \min\prn{\sqrt{T\rankavg}, \sqrt{d-\rankavg}}$\,.
\end{lemma}

\begin{proof}
    We first bound \(\norm{Q\sqprn{\X^{+}\X}}_{F}\) using the operator norm bound on $Q$ (\cref{cor:Q spectrum}):
    \begin{align*}
        \norm{Q\sqprn{\X^{+}\X}}_{F} & \leq \underbrace{\norm{Q}}_{\leq 1} \cdot \norm{\X^{+}\X}_{F} \leq \norm{\X^{+}\X}_{F} 
        =
        \sqrt{\rank\prn{\X^{+}\X}} = \sqrt{T\rankavg}
        \,.
    \end{align*}
    Next, we use a pseudo-inverse property---that $\X^{+}\X \preccurlyeq \I$---and the positivity of $Q$ to show,
    \begin{align*}
        \mathbf{0} 
        & 
        \preccurlyeq Q\sqprn{\I - \X^{+}\X} 
        \notag \\
        Q\sqprn{\X^{+}\X} & \preccurlyeq Q\sqprn{\I} 
        \notag \\
        \norm{Q\sqprn{\X^{+}\X}}_{F} 
        & 
        \leq 
        \norm{Q\sqprn{\I}}_{F} 
        =
        \Bignorm{\frac{1}{T} \sum_{i=1}^{T} \mP_{i}}_{F}
        \leq 
        \frac{1}{T} \sum_{i=1}^{T} \norm{\mP_{i}}_{F} 
        \notag \\
        &= \frac{1}{T} \sum_{i=1}^{T} \sqrt{\rank\prn{\mP_{i}}} 
        = \frac{1}{T} \sum_{i=1}^{T} \sqrt{d - \rank\prn{\X_i}} 
        \notag \\
        \explain{\text{Jensen (concave)}} 
        & 
        \leq \sqrt{d - \rankavg}
        \,.
    \end{align*}
\end{proof}
\newpage
\section{Proofs of Universal Continual Regression Rates (\cref{sec:universal_by_sgd,sec:cl_wor})}
\allowdisplaybreaks

The proofs in this appendix focus on the properties of forgetting and loss, ``translating'' them into the language of last-iterate SGD. 
We then apply our last-iterate results, proved in \cref{app:last-iterate-proofs}.

\subsection{Proof of \cref{thm:cl_by_sgd_main}: A Universal $\bigO(1/\sqrt[4]{k})$ Rate}
\label{sec:cl_by_sgd_proof}
\begin{recall}[\cref{thm:cl_by_sgd_main}]
Under a random ordering with replacement 
over $T$ jointly realizable tasks, 
the expected loss and forgetting of {Schemes~\ref{proc:regression_to_convergence},~\ref{scheme:kaczmarz}}
    after $k\geq 2$ iterations are bounded as,
\begin{align*}
\mathbb{E}_{\tau}\left[\Loss\left(\w_{k}\right)\right]
&
=
\mathbb{E}_{\tau}
\bigg[
\frac{1}{2T} 
\sum_{m=1}^{T} \bignorm{\X_m \w_k - \y_m}^2
\bigg]
\le
\frac{2}{ \sqrt[4]{k}}
\bignorm{\teacher}^2 R^2
\,,
\\
\mathbb{E}_{\tau}
\left[
F_{\tau}(k)
\right]
&
=
\mathbb{E}_{\tau}
\bigg[
\frac{1}{2k}
\sum_{t=1}^{k} 
\norm{\X_{\tau(t)}\w_{k} - \y_{\tau(t)}}^2
\bigg]
\le 
\frac{5}{ \sqrt[4]{k-1}}
\norm{\teacher}^2 R^2
\,.
\end{align*}
\end{recall}

\begin{proof}
    Let $\tau$ be a random with-replacement ordering, and $\w_0, \ldots, \w_k$ be the corresponding iterates produced by the continual \cref{proc:regression_to_convergence} 
    (or the equivalent Kaczmarz \cref{scheme:kaczmarz}).
    By \cref{reduc:kaczmarz}, these are exactly the (stochastic) gradient descent iterates produced given an initialization $\w_0$ and a step size of $\eta=1$, on the loss sequence $f_{\tau(1)}, \ldots, f_{\tau(k)}$, where we defined:
    \begin{align*}
        f_m(\w) \eqq \frac12\norm{\X_m^+\X_m(\w - \teacher)}^2.
    \end{align*}
    Furthermore,
    \cref{lem:cl_gd_equiv} states that for all $\w\in \R^d$,
    \begin{align*}
        \cL(\w) 
        = 
        \frac1{2 T}\sum_{m=1}^T \norm{\X_m \w - \y_m}^2
        =
        \E_{m \sim \Unif([T])} \loss_{m} (\w)
        \leq 
        R^2 \E_{m \sim \Unif([T])}f_m(\w)
        \,.
    \end{align*}
    Therefore, establishing last iterate convergence of with-replacement SGD (\cref{def:sgd_withreplacement}) on the objective function
    \begin{align*}
        \f(\w) \eqq \E_{m\sim [T]} f_m(\w)\,,
    \end{align*}
    will imply the desired result.
    Indeed, again by \cref{lem:cl_gd_equiv}, $f_m(\cdot)$ is $1$-smooth for all $m\in[T]$.
    Hence, plugging in $\A=\X_{m}^{+}\X_{m}\Rightarrow\norm{\A}=1=\beta$
    into \cref{thm:sgd_last_iterate_main}, 
    SGD with $\eta=1$ guarantees that after $k\geq1$ gradient steps: 
    \begin{align*}
        \E \f(\w_k) 
        \leq \frac{e\norm{\w_0 - \teacher}^2}{2\sqrt[4]{k}}
        \leq \frac{2\norm{\w_0 - \teacher}^2}{\sqrt[4]{k}}\,,
    \end{align*}
    and therefore $\E \cL(\w_k)\leq \frac{2 R^2\norm{\w_0 - \teacher}^2}{\sqrt[4]{k}}$, which proves the first claim. 
    The second claim follows immediately from \cref{lem:in_to_all_sample}, and we are done.
\end{proof}

\newpage

\subsection{Proving \cref{thm:cl_by_sgd_wor}: Main Result for Without Replacement Orderings}
\label{sec:cl_wor_proofs}
\begin{recall}[\cref{thm:cl_by_sgd_wor}]
Under a random ordering without replacement 
over $T$ jointly realizable tasks, 
the expected loss and forgetting of {Schemes~\ref{proc:regression_to_convergence},~\ref{scheme:kaczmarz}} after $k\in \{2,\dots,T\}$ iterations are both bounded as,
    \begin{align*}
        \E \sbr{\Loss\prn{\w_k}},\,
        \E \sbr{F_{\tau}(k)}
        \leq 
        \min\left(
        \frac{7}{ \sqrt[4]{k-1}},~
        \frac{d-\rankavg+1}{k-1}
        \right)
        \norm{\teacher}^2 R^2
        \,.
    \end{align*}
\end{recall}
\begin{proof}

From \cref{lem:in_to_all_sample,lem:cl_gd_equiv}, 
we have
$$\expectation_{\tau}\sqprn{F_{\tau}\prn{k}}
\le
\expectation_{\tau}\norm{\X_{\tau\prn{k}}\w_{k-1}-\y_{\tau\prn{k}}}^2
+\frac{\norm{\teacher}^2R^{2}}{k}
\le
2R^2\expectation_{\tau} f_{\tau(k)} (\w_{k-1})
+\frac{\norm{\teacher}^2R^{2}}{k}
\,.$$
Combining with
\cref{prop:loss_to_forgetting}, we get,
\begin{align*}
\mathbb{E}_{\tau}\left[\mathcal{L}\left(\w_{k}\right)\right] 
& =
\frac{k}{T}\mathbb{E}_{\tau}\left[F_{\tau}\left(k\right)\right]
+
\frac{T-k}{2T}
\mathbb{E}_{\tau}\left\Vert \X_{\tau\left(k+1\right)}\w_{k}-\y_{\tau\left(t\right)}\right\Vert ^{2}
\\
&
\le
\frac{k}{T}
\prn{
2R^2\expectation_{\tau} f_{\tau(k)} (\w_{k-1})
+
\frac{\norm{\teacher}^2R^{2}}{k}
}
+
\frac{T-k}{2T}
\expectation_{\tau} f_{\tau(k+1)} (\w_{k})
\\
\explain{k\le T}
&
\le
R^2
\prn{
\frac{2k}{T}
\expectation_{\tau} f_{\tau(k)} (\w_{k-1})
+
\frac{T-k}{T}
\expectation_{\tau} f_{\tau(k+1)} (\w_{k})
}
+\frac{\norm{\teacher}^2R^{2}}{k}
\,.
\end{align*}
Thus, to bound both the expected forgetting and loss, we need to bound expressions like $\expectation_{\tau}\! f_{\tau(k+1)} (\w_{k})$.

\bigskip

We first prove the \emph{dimension dependent} term.
Note that,
\begin{align*}
 2\expectation_{\tau} f_{\tau(k)} (\w_{k-1})
 = 
 \expectation_{\tau}\norm{\X_{\tau\prn{k}}^{+}\X_{\tau\prn{k}}\prn{\w_{k-1}-\teacher}}^2
\triangleq
\expectation_{\tau}\norm{\prn{\I-\mP_{\tau\prn{k}}}\prn{\w_{k-1}-\teacher}}^2.
\end{align*}
Recall that from \cref{eq:projected_errors} in the proof of \cref{lem:in_to_all_sample}, we have
\[
\prn{\w_{k-1}-\teacher}=\mP_{\tau\prn{k-1}}\cdots\mP_{\tau\prn{1}}\prn{\w_{0}-\teacher}=-\mP_{\tau\prn{k-1}}\cdots\mP_{\tau\prn{1}}\teacher.
\]
Thus, we obtain 
\begin{align*}
&\expectation_{\tau}\norm{\prn{\I-\mP_{\tau\prn{k}}}\prn{\w_{k-1}-\teacher}}^2
=
\expectation_{\tau}\norm{\prn{\I-\mP_{\tau\prn{k}}}\mP_{\tau\prn{k-1}}\cdots\mP_{\tau\prn{1}}\teacher}^2
\\
 & 
 \le\expectation_{\tau}\norm{\prn{\I-\mP_{\tau\prn{k}}}\mP_{\tau\prn{k-1}}\cdots\mP_{\tau\prn{1}}}_{2}^{2}\cdot\norm{\teacher}^2
 \leq\norm{\teacher}^2\expectation_{\tau}\norm{\prn{\I-\mP_{\tau\prn{k}}}\mP_{\tau\prn{k-1}}\cdots\mP_{\tau\prn{1}}}_{F}^{2}
 \\
 & 
 =
 \norm{\teacher}^2\expectation_{\tau}\tr\prn{\mP_{\tau\prn{1}}\cdots\mP_{\tau\prn{k-1}}\prn{\I-\mP_{\tau\prn{k}}}\mP_{\tau\prn{k-1}}\cdots\mP_{\tau\prn{1}}}.
\end{align*}

By exchangeability,
\begin{align*}
    &\tr\prn{\mP_{\tau\prn{1}}\cdots\mP_{\tau\prn{t-1}}\prn{\I-\mP_{\tau\prn{t}}}\mP_{\tau\prn{t-1}}\cdots\mP_{\tau\prn{1}}}
    \\
    &=\tr\prn{\mP_{\tau\prn{t}}\cdots\mP_{\tau\prn{2}}\prn{\I-\mP_{\tau\prn{1}}}\mP_{\tau\prn{2}}\cdots\mP_{\tau\prn{t}}}\,.
\end{align*}


Let us define 
$
a_{t}=\tr\prn{\mP_{\tau\prn{t}}\cdots\mP_{\tau\prn{2}}\prn{\I-\mP_{\tau\prn{1}}}\mP_{\tau\prn{2}}\cdots\mP_{\tau\prn{t}}}.
$
Then, we have
\begin{align*}
a_{t+1} 
& =
\tr\prn{\mP_{\tau\prn{t+1}}\cdots\mP_{\tau\prn{2}}\prn{\I-\mP_{\tau\prn{1}}}\mP_{\tau\prn{2}}\cdots\mP_{\tau\prn{t+1}}}
\\
\explain{\text{cyclic property of trace}} 
& =
\tr\prn{\mP_{\tau\prn{t+1}}^{2}\mP_{\tau\prn{t}}\cdots\mP_{\tau\prn{2}}\prn{\I-\mP_{\tau\prn{1}}}\mP_{\tau\prn{2}}\cdots\mP_{\tau\prn{t}}}
\\
\explain{\text{Von Neumann's trace inequality}} 
& \leq
\underbrace{\bignorm{\mP_{\tau\prn{t+1}}^{2}}_{2}}_{=1}
\tr\prn{\mP_{\tau\prn{t}}\cdots\mP_{\tau\prn{2}}\prn{\I-\mP_{\tau\prn{1}}}\mP_{\tau\prn{2}}\cdots\mP_{\tau\prn{t}}}=a_{t}
\,,
\end{align*}
showing $(a_t)_t$ is a non-increasing sequence.
Thus, for all $k\ge2$,
\begin{align*}
&
2
\expectation_{\tau} f_{\tau(k)} (\w_{k-1})
=
\expectation_{\tau}\norm{\prn{\I-\mP_{\tau\prn{k}}}\w_{k-1}}^2
\le
\norm{\teacher}^2\expectation_{\tau}a_{k}
\le
\tfrac{\norm{\teacher}^2}{k-1}\tsum_{t=2}^{k}\expectation_{\tau}a_{t}
\\
& =
\tfrac{\norm{\teacher}^2}{k-1}\sum_{t=2}^{k}\expectation_{\tau}\sqprn{\tr\prn{\mP_{\tau\prn{t}}\cdots\mP_{\tau\prn{2}}\cdots\mP_{\tau\prn{t}}}
-
\tr\prn{\mP_{\tau\prn{t}}\cdots\mP_{\tau\prn{1}}\cdots\mP_{\tau\prn{t}}}}
\\
\explain{\text{exchangeability}} 
& =\tfrac{\norm{\teacher}^2}{k-1}\sum_{t=2}^{k}\expectation_{\tau}\sqprn{\tr\prn{\mP_{\tau\prn{t-1}}\cdots\mP_{\tau\prn{1}}\cdots\mP_{\tau\prn{t-1}}}-\tr\prn{\mP_{\tau\prn{t}}\cdots\mP_{\tau\prn{1}}\cdots\mP_{\tau\prn{t}}}}
\\
\explain{\text{telescoping}} 
& =\tfrac{\norm{\teacher}^2}{k-1}\expectation_{\tau}\sqprn{\tr\prn{\mP_{\tau\prn{1}}}-\tr\prn{\mP_{\tau\prn{k}}\cdots\mP_{\tau\prn{1}}\cdots\mP_{\tau\prn{k}}}}
\\
& \le\tfrac{\norm{\teacher}^2}{k-1}\expectation_{\tau}\sqprn{\tr\prn{\mP_{\tau\prn{1}}}}=\frac{\norm{\teacher}^2\prn{d-\rankavg}}{k-1}.
\end{align*}

For the \textbf{second}, \emph{parameter independent} term,
note that from \cref{lem:cl_gd_equiv}, $f_m(\cdot)$ is $1$-smooth for all $m\in[T]$, and recall that the iterates $\w_t$ follow SGD dynamics with $\eta=1$ (\cref{reduc:kaczmarz}).
Hence, by \cref{lem:wor_last_before_stab}, without-replacement SGD with $\beta=\eta=1$ guarantees that after $k\geq1$ gradient steps:
\begin{align*}
\expectation_{\tau} f_{\tau(k)} (\w_{k-1})
\leq 
\frac{e\cdot\norm{\teacher}^2}{\sqrt[4]{k-1}}.
\end{align*}

Plugging in the (monotonic decreasing) bounds that we just derived in the inequalities from the beginning of this proof, we get
\begin{align*}
\expectation_{\tau}\sqprn{F_{\tau}\prn{k}}
&\le
2R^2\expectation_{\tau} f_{\tau(k)} (\w_{k-1})
+\frac{\norm{\teacher}^2R^{2}}{k}
\\
&
\le
R^2
\min 
\prn{
\frac{2e\norm{\teacher}^2}{\sqrt[4]{k-1}},\,
\frac{\norm{\teacher}^2\prn{d-\rankavg}}{k-1}
}
+\frac{\norm{\teacher}^2R^{2}}{k}
\\
&
\le
\min 
\prn{
\frac{7}{\sqrt[4]{k-1}},\,
\frac{d-\rankavg+1}{k-1}
}
\norm{\teacher}^{2}R^{2}
\,,
\\
&
\\
\mathbb{E}_{\tau}\left[\mathcal{L}\left(\w_{k}\right)\right] 
&
\le
R^2
\prn{
\frac{k}{T}
2\expectation_{\tau} f_{\tau(k)} (\w_{k-1})
+
\frac{T-k}{2T}
2\expectation_{\tau} f_{\tau(k+1)} (\w_{k})
}
+\frac{\norm{\teacher}^2R^{2}}{k}
\\
&
\le
\prn{
\frac{k}{T}
+
\frac{T-k}{2T}
}
\min 
\prn{
\frac{2e}{\sqrt[4]{k-1}},\,
\frac{d-\rankavg}{k-1}
}
\norm{\teacher}^{2}R^{2}
+\frac{\norm{\teacher}^{2}R^{2}}{k}
\\
&
=
\frac{T+k}{2T}
\min 
\prn{
\frac{2e}{\sqrt[4]{k-1}},\,
\frac{d-\rankavg}{k-1}
}
\norm{\teacher}^{2}R^{2}
+\frac{\norm{\teacher}^{2}R^{2}}{k}
\\
\explain{k\le T}
&
\le
\min 
\prn{
\frac{7}{\sqrt[4]{k-1}},\,
\frac{d-\rankavg+1}{k-1}
}
\norm{\teacher}^{2}R^{2}
\,.
\end{align*}
\end{proof}

\newpage
\section{Proofs of Last-Iterate SGD Bounds (\cref{sec:sgd})}
\label{app:last-iterate-proofs}
In this section we provide proofs and full technical details of our upper bounds for least squares SGD.
We begin by recording a few elementary well-known facts, which can be found in e.g., \citet{Bubeck2015}. 
We provide proof for completeness. 
\begin{lemma}[Fundamental regret inequality for gradient descent]\label{lem:gd_funadmental_inequality}
Let $\w_0\in \R^d, \eta>0$, and suppose $\w_{t+1} = \w_t - \eta \vg_t$ for all $t$, where $\vg_0,\ldots,\vg_T \in \R^d$ are arbitrary vectors. Then for any $\tilde \w \in \R^d$ it holds that:
\begin{align*}
    \sum_{t=0}^T \vg_t\T(\w_t-\tilde \w)
    \leq
    \frac{\|\w_0-\tilde \w\|^2}{2\eta} 
    + \frac\eta2 \sum_{t=0}^T \|\vg_t\|^2
    .
\end{align*}
\end{lemma}
\begin{proof}
    Observe,
    \begin{align*}
        \norm{\w_{t+1} - \tilde \w}^2
        &=
        \norm{\w_{t} - \tilde \w}^2
        -2\eta\vg_t\T(\w_t - \tilde \w)
        + \eta^2\norm{\vg_t}^2
        \\
        \iff
        \vg_t\T(\w_t - \tilde \w)
        &=\frac{1}{2\eta} \br{\norm{\w_{t} - \tilde \w}^2
        - \norm{\w_{t+1} - \tilde \w}^2
        } + \frac\eta2\norm{\vg_t}^2.
    \end{align*}
    Summing the above over $t=0,\ldots,T$ and telescoping the sum leads to,
    \begin{align*}
        \sum_{t=0}^T\vg_t\T(\w_t - \tilde \w)
        &=\frac{1}{2\eta} \br{\norm{\w_{0} - \tilde \w}^2
        - \norm{\w_{T+1} - \tilde \w}^2
        } + \frac\eta2\sum_{t=0}^T\norm{\vg_t}^2
        \\
        &\leq\frac{\norm{\w_{0} - \tilde \w}^2}{2\eta}
        + \frac\eta2\sum_{t=0}^T\norm{\vg_t}^2,
    \end{align*}
    which completes the proof.
\end{proof}
\begin{lemma}[Descent lemma]
\label{lem:gd_descent}
    Let $f : \R^d \to \R$ be $\beta$-smooth for $\beta>0$, and suppose $\min_\w f(\w)\in \R$ is attained.
    Then, for any $\eta>0$, $\w\in \R^d$, we have for $\w^+ = \w - \eta \nabla f(\w)$:
    \begin{align*}
        f(\w^+) \leq f(\w) - \eta \br{1-\frac{\eta\beta}2} \|\nabla f(\w)\|^2.
    \end{align*}
    Furthermore, for any $\teacher\in \argmin_\w f(\w)$, it holds that:
    \begin{align*}
        \|\nabla f(\w)\|^2 \leq 2\beta\br{f(\w) - f(\teacher)}.
    \end{align*}
\end{lemma}
\begin{proof}
    Observe, by $\beta$-smoothness:
    \begin{align*}
    f(\w^+)
    &\leq
    f(\w) + \nabla f(\w) \cdot (\w^+-\w) 
        + \frac\beta2 \|\w^+-\w\|^2
    \\
    &=
    f(\w) - \eta \nabla f(\w) \cdot \nabla f(\w) + \frac\beta2 \eta^2 \|\nabla f(\w)\|^2
    \\
    &=
    f(\w) - \eta \br{1-\frac{\eta\beta}2} \|\nabla f(\w)\|^2
    ,
    \end{align*}
    which proves the first claim.
    For the second claim, apply the above inequality with $\eta=1/\beta$, which gives
    \begin{align*}
    f(\w^+)
    &\leq
    f(\w) - \frac{1}{2\beta}\|\nabla f(\w)\|^2
    \\
    \iff
    \|\nabla f(\w)\|^2
    &\leq 2\beta \br{f(\w) - f(\w^+)}
    .
    \end{align*}
    The second claim now follows by using the fact that 
    $f(\teacher) \leq f(\w^+)$.
\end{proof}
\subsection{Proofs for With Replacement Orderings}
\label{app:sgd_with}
As discussed in the main text, our results hold for a wider range of step sizes compared to the classical SGD bounds in the smooth realizable setting. 
This is enabled due to the following lemma.
\begin{lemma}\label{lem:sqloss_gradineq}
    Assume that $f(\w) = \frac12 \norm{\A \w-\vb}^2$ for some matrix $\A$ and vector $\vb$, and let $\teacher\in \R^d$ be such that $f(\teacher)=0$.
    Then, we have:
    \begin{align*}
    2 f(\w) 
    &= \nabla f(\w)\T (\w - \teacher)\,,
    \end{align*}
    and for any $\vz\in \R^d$ and $\gamma>0$:
    \begin{align*}
        (2-\gamma)f(\w) - \frac1\gamma f(\vz)
    &\leq \nabla f(\w)\T (\w - \vz)\,.
    \end{align*}
\end{lemma}

\begin{proof} 
    For any $\w\in \R^d$, since $\A \teacher=\vb$ and $f(\w)=\frac{1}{2}\|\A(\w-\teacher)\|^2$, we have:
    \begin{align*}
        \nabla f(\w)\T (\w - \vz)
        &= \abr[b]{\A\T \A(\w-\teacher) , \w - \vz}
        \\&=
        \abr[b]{\A\T \A(\w-\teacher) , \w - \teacher}
        -\abr[b]
        {\A\T \A(\w-\teacher) , \vz - \teacher}
        \\&=
        \abr[b]{\A\w-\vb , \A\w - \vb}
        -\abr[b]
        {\A\w- \vb, \A\vz - \vb}
        \\ &=
        2f(\w) -\abr[b]
        {\A\w- \vb, \A\vz - \vb}.
    \end{align*}
    Plugging in $\vz=\teacher$, the second term vanishes
    (since $\A\teacher-\vb=\vb-\vb=\0$)
    and the first claim follows.
    For the second claim, note that by Young's inequality:
    \begin{align*}
    \nabla f(\w)\T (\w - \vz)
    &=
2f(\w) -\abr[b]
        {\A\w- \vb, \A\vz - \vb}
        \\
        &
        \geq 
        2f(\w) -\frac{\gamma}{2}\norm{\A\w-\vb}^2
        - \frac{1}{2\gamma}\norm{\A \vz-\vb}^2
        =
        (2-\gamma)f(\w) 
        - \frac1\gamma f(\vz)\,.
    \end{align*}
\end{proof}

\newpage

\begin{recall}[\cref{lem:regret_bound}]
Consider the $\beta$-smooth, realizable \cref{setup:sgd_main},
and let $T\geq 1$, $(i_0, \ldots, i_{T}) \in \cI^{T+1}$ be an arbitrary sequence of indices in $\cI$, 
and ${\w_0\in \R^d}$ 
be an arbitrary initialization.
Then, 
the gradient descent iterates given by
${\w_{t+1} \gets \w_t - \eta \nabla f(\w_t; i_t)}$ for a step size $\eta<2/\beta$,
hold:
    \begin{align*}
        \sum_{t=0}^T f\left(\w_t; i_t\right)
        \leq 
        \frac{\norm{\w_0 - \teacher}^2}{2\eta(2-\eta\beta)}
        \,.
    \end{align*}
\end{recall}

\medskip

\begin{proof}
Denote $f_t(\w)\eqq f(\w;i_t)$, and observe
    by \cref{lem:gd_funadmental_inequality};
    \begin{align*}
    &
        \sum_{t=0}^T 
        \abr{\nabla f_t(\w_t), \w_t - \teacher}
        \leq
        \frac{\norm{\w_0 - \teacher}^2}{2\eta}
        + \frac\eta2\sum_{t=0}^T \norm{\nabla f_t(\w_t)}^2
        \\
        &\leq
        \frac{\norm{\w_0 - \teacher}^2}{2\eta}
        + \eta\beta\sum_{t=0}^T f_t(\w_t) - f_t(\teacher)
        =
        \frac{\norm{\w_0 - \teacher}^2}{2\eta}
        + \eta\beta\sum_{t=0}^T f_t(\w_t)\,, 
    \end{align*}
    where the second inequality follows from \cref{lem:gd_descent}.
    On the other hand, by \cref{lem:sqloss_gradineq},
    \begin{align*}
\sum_{t=0}^T 
    \abr{\nabla f_t(\w_t), \w_t - \teacher}
        = 
        \sum_{t=0}^T 
        2 f_t(\w_t)\,.
    \end{align*}
    Combining the two displays above, it follows that
    \begin{align*}
        (2-\eta\beta)\sum_{t=0}^T f_t(\w_t)
        \leq 
        \frac{\norm{\w_0
        - \teacher}^2}{2\eta}\,,
    \end{align*}
    and the result follows after dividing by $(2-\eta\beta)$.
\end{proof}

\newpage

\begin{recall}[\cref{lem:ratio_avg_last}]
Consider the $\beta$-smooth, realizable \cref{setup:sgd_main}. 
Let $T\geq 1$.
Assume $\cP$ is a distribution over $\cI^{T+1}$ such that for every $0\leq t\leq\tau_1 \leq \tau_2 \leq T$, the following holds:
For any ${i_0,\ldots i_{t-1}\in\cI^{t}}, i\in \cI$, 
$\Pr(i_{\tau_1}=i|i_0,\ldots,i_{t-1})=\Pr(i_{\tau_2}=i|i_0,\ldots,i_{t-1})$.
Then, for any initialization $\w_0\in \R^d$, 
with-replacement SGD (\cref{def:sgd_withreplacement}) with step-size $\eta<2/\beta$, holds:
    \begin{align*}
     \E f(\w_T,i_T)\leq (eT)^{\eta\beta\left(1-\eta\beta/4\right)}\E \left[\tfrac{1}{T+1}\tsum_{t=0}^T f(\w_t;i_t)\right],
    \end{align*}
    where the expectation is taken with respect to $i_0,\ldots,i_T$ sampled from $\cP$.
\end{recall}

\bigskip

\begin{proof} 
Denote $f_t(\w)\!\eqq\! f(\w;i_t),\,\vg_t\!\eqq\!\nabla f_t(\w_t)$, and observe that
    by \cref{lem:gd_funadmental_inequality}, 
    $\forall \vz\in \R^d,\,t\leq T$ (w.p.~$1$):
\begin{align*}
	\sum_{t=T-k}^T \abr{\vg_t, \w_t - \vz}
	&\leq \frac{\norm{\w_{T-k} - \vz}^2}{2\eta}
	+ \frac\eta2\sum_{t=T-k}^T \norm{\vg_t}^2
	\\
    \explain{\text{Descent \cref{lem:gd_descent}}}
	&\leq \frac{\norm{\w_{T-k} - \vz}^2}{2\eta}
	+ \eta\beta\sum_{t=T-k}^T f_t(\w_t) - f_t(\teacher)
	\\
	&= \frac{\norm{\w_{T-k} - \vz}^2}{2\eta}
	+ \eta\beta\sum_{t=T-k}^T f_t(\w_t) - f_t(\vz) + f_t(\vz) - f_t(\teacher)\,.
\end{align*}
By \cref{lem:sqloss_gradineq}, this implies for any $\gamma>0$:
\begin{align*}
    &
    \sum_{t=T-k}^T (2-\gamma -\eta\beta)f_t(\w_t) - \br{\frac{1}{\gamma}-\eta\beta}f_t(\vz)
    \\
    &
    =
    \sum_{t=T-k}^T 
    \!
    \Bigprn{
    (2-\gamma)f_t(\w_t) - 
    \frac{1}{\gamma}f_t(\vz)
    }
    +
    \eta\beta
    \!
    \sum_{t=T-k}^T 
    f_t(\vz) - f_t(\w_t)
    \\
    \explain{\text{\cref{lem:sqloss_gradineq}}}
    &
    \leq 
    \sum_{t=T-k}^T
    \abr{\vg_t, \w_t - \vz}
	+ \eta\beta\sum_{t=T-k}^T f_t(\vz) - f_t(\w_t)
    \\
    \explain{\text{above}}
    &
    \leq 
    \frac{\norm{\w_{T-k} - \vz}^2}{2\eta}
	+ \eta\beta\sum_{t=T-k}^T f_t(\vz) - \underbrace{f_t(\teacher)}_{=0}
    \\
    \implies
    (2-\gamma -\eta\beta)
    \!
    \sum_{t=T-k}^T 
    \!
    f_t(\w_t) 
    &\leq 
    \frac{\norm{\w_{T-k} - \vz}^2}{2\eta}
	+ \frac{1}{\gamma}\sum_{t=T-k}^T f_t(\vz)
    \,.
\end{align*}
Now, set $\vz=\w_{T-k}$ 
and take expectations to obtain:
\begin{align*}
    (2-\gamma -\eta\beta)
    \sum_{t=T-k}^T \E f_t(\w_t) 
    &\leq 
    0 + 
    \frac{1}{\gamma}
    \sum_{t=T-k}^T \E f_t(\w_{T-k})
    \\
    \frac{1}{k+1}
    \sum_{t=T-k}^T 
    \E f_t(\w_t) 
    &
    \leq 
    \frac{1}{\bigprn{k+1}
    \gamma (2-\gamma -\eta\beta)
    }
    \sum_{t=T-k}^T \E f_t(\w_{T-k})\,.
\end{align*}
Defining $S_k \eqq 
\frac{1}{k+1}\sum_{t=T-k}^T f_t(\w_t)$,
implies that
\begin{align*}
    (k+1) S_k -  k S_{k-1}
    &= \sum_{t=T-k}^T  f_t(\w_t)
    - \sum_{t=T-k+1}^T  f_t(\w_t)
    =  f_{T-k}(\w_{T-k})\,,
\end{align*}
and 
by the assumption on the distribution $\cP$ it follows that
$\E f_{T-k}(\w_{T-k}) = \E f_t(\w_{T-k})$ for any $t\geq T-k$.

Thus, combined with our previous display, 
\begin{align*}
    \E S_k
    &\leq \frac{1}{\bigprn{k+1}\gamma (2-\gamma -\eta\beta)}  \sum_{t=T-k}^T \E f_t(\w_{T-k})
    \\&
    = \frac{1}{\bigprn{k+1}
    \gamma (2-\gamma -\eta\beta)
    } \sum_{t=T-k}^T 
    \Bigprn{(k+1)\E S_k - k\E S_{k-1}}
    \\&
    = 
    \frac{1}{\gamma (2-\gamma -\eta\beta)}
    \Bigprn{(k+1)\E S_k - k\E S_{k-1}}\,.
\end{align*}
Rearranging, 
denoting $\displaystyle c\triangleq\gamma (2-\gamma -\eta\beta)$,
and requiring $c\in (0,1)$,
we get
\begin{align}
\frac{k}{c}
\E S_{k-1}
&\leq 
\prn{
\frac{k+1}{c}-1
}\E S_k
\notag
\\
\iff
\E S_{k-1}
&\leq 
\frac{k+1-c}{k}
\E S_k
\notag
\\
\implies
\E f_T(\w_T)
=
\E S_{0}
&\leq 
\prod_{k=1}^{T}
\prn{
1+
\frac{1-c}{k}
}\E S_T
\notag
\\
\explain{1+x\le e^x,\forall x\ge 0}
&
\le 
\exp \prn{
\sum_{k=1}^{T}
\frac{1-c}{k}
}
\E S_T 
\notag
\\
&
=
\exp \prn{
\prn{1-c}
\sum_{k=1}^{T}
\frac{1}{k}
}
\cdot
\E S_T 
\le
\exp \bigprn{
\prn{1-c}
\prn{1+\log T}
} \E S_T
\notag
\\
&
=
\prn{eT}^{1-c}
\cdot
\E\sbr{\frac{1}{T+1}\sum_{t=0}^T f_t(\w_t)}\,.
\label{eq:remarked_sgd}
\end{align}
Now, getting the ``best'' rate requires maximizing $c=\gamma (2-\gamma -\eta\beta)$.
To this end, we choose $\gamma = 1-\frac{\eta\beta}{2}$,
which implies
$c=\bigprn{1-\frac{\eta\beta}{2}}^2$
(under the $\eta<\frac{2}{\beta}$ condition, we now have both $\gamma>0$ and $c\in(0,1)$ as required above).
Then, 
$1-c = \eta\beta\bigprn{1-\frac{\eta\beta}{4}}$, and we finally get the required
\begin{align*}
\E f_T(\w_T)
&
\le
\prn{eT}^{\eta\beta\bigprn{1-\frac{\eta\beta}{4}}}
\cdot
\frac{1}{T+1}\sum_{t=0}^T f_t(\w_t)\,.
\end{align*}
\end{proof}

\newpage

\subsection{Extending the SGD Bounds to Without Replacement Orderings} 

Here, we extend \cref{thm:sgd_last_iterate_main} to a \emph{without}-replacement setting.
Specifically, we consider gradient descent under a random \emph{permutation} of the $T$ tasks. 
That is, for some initialization $\w_0\in \R^d$, step size $\eta>0$, and $\perm_t \sim \Unif(\cI)$ sampled without replacement,
\begin{align}\label{def:sgd_wor}
    \w_{t+1} \gets \w_t - \eta \nabla f(\w_t; \perm_t)\,,
\end{align}
where $f(\w; i) \eqq \frac12\norm{\A_i \w - \vb_i}^2$ as defined in \cref{setup:sgd_main}.
Our main result is given below.
\bigskip
\begin{theorem}Last-Iterate Bound for Realizable Regression Without Replacement
\label{thm:wor_sgd_last_iterate_main}
    Consider the $\beta$-smooth, realizable \cref{setup:sgd_main}.
    Define for all $T\!\geq\!2$, 
    $\hat f_{0:T}(\w) \!\eqq\! \frac1{T+1}\sum_{t=0}^{T} f(\w; \perm_t)$.
    Then, without-replacement SGD (\eqref{def:sgd_wor}) with step-size $\eta< \hfrac{2}{\beta}$,~holds:
$$
    \E_\perm \hat f_{0:T}(\w_T) 
    	\leq 
            \frac{e D^2}{
            \eta (2-\eta\beta)
            T^{1-\eta\beta\left(1-\eta\beta/4\right)}}
            +
            \frac{4\beta^2\eta D^2}{T}\,,\quad
            \forall T = 2, \dots, n-1
            \,,
$$
where $D\eqq\norm{\w_0 - \teacher}$. 
In particular,
for $\eta=\frac{1}{\beta\log T}$
yields
${\frac{14\beta D^2 \log T}{T}}$
and
$\eta=\frac{1}{\beta}$
yields
$\frac{7\beta D^2}{\sqrt[4]{T}}$.
%
%
%
%
%
%
%
\end{theorem}
The proof, given next, is based on the algorithmic stability of SGD \citep{bousquet2002stability,shalev2010learnability,hardt2016train}, and more specifically, on a variant of stability, suitable for without replacement sampling \citep{sherman2021optimal,koren2022benign}.


\bigskip

\label{app:sgd_without}
The proof of our theorem follows by a combination of \cref{lem:wor_last_before_stab,lem:wor_generalization}. 
The first, stated below, establishes a bound on the expected ``next sample'' loss and follows immediately by combining \cref{lem:ratio_avg_last,lem:regret_bound}
(notice that
$\eta<\frac{2}{\beta}\Longrightarrow
\exp\bigprn{\eta\beta\bigprn{1-\frac{\eta\beta}{4}}} 
\mapsto
\exp\bigprn{z\bigprn{1-\frac{z}{4}}}$ for $z\in (0,2)$,
which is monotonic increasing and upper bounded by
$e$).

\bigskip

\begin{lemma}\label{lem:wor_last_before_stab} 
For any step-size $\eta< 2/\beta$ and initialization $\w_0\in \R^d$, without-replacement SGD \cref{def:sgd_wor} satisfies, for all $1\leq T\leq n-1$:
\begin{align*}
    \E_\perm f(\w_T; \perm_T)
    &\leq e^{\eta\beta\left(1-\frac{\eta\beta}{4}\right)}T^{\eta\beta\left(1-\frac{\eta\beta}{4}\right)}\E_\perm\left[\frac{1}{T+1}\sum_{t=0}^T f(\w_t;\perm_t)\right]
    \leq 
    \frac{e\cdot\norm{\w_0 - \teacher}^2}{
    2\eta(2-\eta\beta) 
    T^{1-\eta\beta\left(1-\frac{\eta\beta}{4}\right)}}
    \,.
\end{align*}
\end{lemma}

    \medskip
    
    Next, we consider the ``empirical loss'' objective. Given any permutation $\perm\in \cI \leftrightarrow \cI$, define:
\begin{align*}
    \hat f_{0:t}(\w) \eqq \frac1{t+1}\sum_{i=0}^t f(\w; \perm_i).
\end{align*}
In the without-replacement setup, our optimization objective is the expected empirical loss $\E_\pi \hat f_{0:t}(\w)$, which, when $t=n$, satisfies $\E_\pi \hat f_{0:t}(\w) = \E_\pi \f(\w)$.
Our second lemma (given next) bounds the expected empirical loss w.r.t.~the next sample loss. This is the crux of extending our with-replacement upper bound to the without-replacement setup.

\newpage

\begin{lemma}\label{lem:wor_generalization}
For without-replacement SGD \cref{def:sgd_wor} with step size $\eta\leq2/\beta$, for all $1\leq T\leq n$, we have that the following holds:
\begin{align*}
    	\E_\perm \hat f_{0:T}( \w_T)
    	&\leq 
            2 \E_\perm f(\w_T; \perm_T)
            +
            \frac{4\beta^2\eta\norm{\w_0 - \teacher}^2}{T+1}
            .
    \end{align*}
\end{lemma}

\bigskip

The proof of \cref{lem:wor_generalization} builds on an algorithmic stability argument similar to that given in \citet{lei20stab}, combined with the without-replacement stability framework proposed by \citet{sherman2021optimal}.
Before turning to the proof given in the next subsection, we quickly prove \cref{thm:wor_sgd_last_iterate_main}.

\begin{proof}[of \cref{thm:wor_sgd_last_iterate_main}] 
    By \cref{lem:wor_generalization,lem:wor_last_before_stab},
    \begin{align*}
        \E_\perm \hat f_{0:T}( \w_T)
        &\leq 
        2 \E_\perm f(\w_T; \perm_T)
        +
        \tfrac{4\beta^2\eta\norm{\w_0 - \teacher}^2}{T+1}
        \leq 
        \tfrac{e\cdot \norm{\w_0 - \teacher}^2}{\eta(2-\eta\beta) }T^{\eta\beta\left(1-\frac{\eta\beta}{4}\right)-1}
        +
        \tfrac{4\beta^2\eta\norm{\w_0 - \teacher}^2}{T+1}.
    \end{align*}
    The result for $\eta=\frac{1}{\beta}$ is straightforward.
    To see the result for $\eta=\frac{1}{\beta\log T}$,
    notice that in this case,
    \begin{align*}
    \tfrac{e D^2 T^{\eta\beta\left(1-\eta\beta/4\right)-1}}{\eta(2-\eta\beta)}
    =
    \frac{e \beta D^2 \log T}{T(2\!-\!\frac{1}{\log T}) }
    T^{\frac{1}{\log T}\left(1-\frac{1}{4\log T}\right)}
    =
    \tfrac{\beta D^2 \log T}{T}
    \frac{\exp\left(2-\frac{1}{4\log T}\right)}{2-\frac{1}{\log T}}
    \le
    \frac{10 \beta D^2 \log T}{T}
    .
    \end{align*}
\end{proof}


\subsubsection{Proving \cref{lem:wor_generalization}} 
\paragraph{Notation.}
We first add a few definitions central to our analysis.
Given a permutation $\pi \in \cI \leftrightarrow \cI$, denote:
\begin{align*}
    \perm(\swap{j}{k}) 
    &\eqq \perm \text{ after swapping the }j\nth\text{ and }k\nth\text{ coordinates}, 
    \\
    \w_\tau^{\perm}
    &\eqq \text{The iterate of SGD on step } \tau \text{ when run on permutation } \perm.
\end{align*}
Most commonly, we will use the following special case of the above:
\begin{align*}
    \w_\tau^{\perm(\swapit)}
    &\eqq \text{The iterate of SGD on step } \tau \text{ when run on } \perm(\swapit).
\end{align*}
When clear from context, we omit $\pi$ from the superscript and simply write $\w_\tau^{(\swapit)}$.
Concretely, these definitions imply $\w_{0}^{(i\leftrightarrow t)} \eqq \w_0$, and $\forall i,t,
\tau \in \cI$,
\begin{align*}
    \w_{\tau+1}^{(i\leftrightarrow t)} 
    &= 
    \w_{\tau}^{(\swapit)}  
    - \eta \nabla 
    f\Bigprn{\w_{\tau}^{(i\leftrightarrow t)} ; 
        \perm(i \leftrightarrow t)_\tau}
    \,.
\end{align*}
\newpage

We have the following important relation, to be used later in the proof.
\begin{lemma}\label{lem:permutation_equiv}
    For all $i,t,\tau \in \cI, 
    i\leq \tau \leq t$, we have:
    \begin{align*}
        \E_\pi f(\w_\tau; \pi_i)
        =
        \E_\pi f(\w_\tau^{(\swapit)}; \perm(\swapit)_i)\,.
\end{align*}
\end{lemma}
\begin{proof}
    The proof follows from observing that the random variables ${f(\w_\tau; \pi_i)}$ and $f(\w_\tau^{(\swapit)}; \perm(\swapit)_i)$ are distributed identically (the indices $\pi_i,\pi_t$ are exchangeable).
    Formally, let $\Pi(\cI) \eqq \cbr{\pi \in \cI \leftrightarrow \cI}$ be the set of all permutations over $\cI$, and observe
    \begin{align*}
	\E_\perm f(\w_\tau^{(\swapit)}; \perm(\swapit)_i)
        =
        \frac{1}{|\Pi(\cI)|}\sum_{\perm \in \Pi(\cI)}
        f(\w_\tau^{\perm(\swapit)}; \perm(\swapit)_i)\,.
    \end{align*}
    On the other hand, 
    \begin{align*}
	\E_{\perm} f(\w_\tau, \pi_i)
        =
        \frac{1}{|\Pi(\cI)|}\sum_{\perm \in \Pi(\cI)}
        f(\w_\tau^{\perm}; \perm_i)\,.
    \end{align*}
    Hence, since there is a one-to-one correspondence between $\perm$ and $\perm(\tau\leftrightarrow i)$, in particular,
    \begin{align*}
        \cbr{\perm \mid \perm \in \Pi(\cI)}
        =
        \cbr{\perm(\swapit) \mid \perm \in \Pi(\cI)},
    \end{align*}
    the result follows.
\end{proof}

\bigskip

Our next lemma, originally given in \citet[][Lemma 2 therein]{sherman2021optimal}, can be thought of as a without-replacement version of the well known stability $\iff$ generalization argument of the with-replacement sampling case \citep{shalev2010learnability,hardt2016train}.
\begin{lemma}
\label{lem:stab_gen_wor}
The iterates of without-replacement SGD \cref{def:sgd_wor}, satisfy for all $t$:
    \begin{align*}
        \E_\perm \sbr{f(\w_t;\perm_t) - \hat f_{0:t-1}(\w_t)}
        = \frac1{t}\sum_{i=0}^{t-1} \E_\perm\sbr{
            f(\w_t;\perm_t) - f(\w_t^{(\swapit)};\perm_t)
        }
    \end{align*}
\end{lemma}
\begin{proof}
    We have, by definition of $\hat f_{0:t-1}$ and \cref{lem:permutation_equiv}:
    \begin{align*}
        \E_\perm \sbr{\hat f_{0:t-1}(\w_t)}
        &= \frac1{t}\sum_{i=0}^{t-1} 
            \E_\perm\sbr{ f(\w_t;\perm_i)
        }
        \\
        &
        = \frac1{t}\sum_{i=0}^{t-1} \E_\perm\sbr{
            f(\w_t^{(\swapit)}; \perm(\swapit)_i)
        }
        = \frac1{t}\sum_{i=0}^{t-1} \E_\perm\sbr{
            f(\w_t^{(\swapit)}; \perm_t)
        }
        ,
    \end{align*}
    where the last equality is immediate since by definition, $\perm(\swapit)_i=\perm_t$.
    The claim now follows by linearity of expectation.
\end{proof}

\newpage

We are now ready to prove our main lemma.
We note that the proof shares some features with that of the with-replacement case (\cref{lem:wr_stability}).
\begin{proof}[of \cref{lem:wor_generalization}]
We prove the theorem for every $t$.
Any $\beta$-smooth realizable function $h:\R^d\to\R_{\ge 0}$ holds that
\begin{align}
\av{h(\tilde \w) - h(\w)}
&
\leq \av{\nabla h(\w)\T(\tilde \w - \w)} + \frac\beta2\norm{\tilde \w - \w}^2
\notag
\\
\explain{\text{Young's ineq.}}
&
\leq \frac1{2\beta}\norm{\nabla h(\w)}^2 + \frac\beta2\norm{\tilde \w - \w}^2 + \frac\beta2\norm{\tilde \w - \w}^2
\notag
\\
&
\leq h(\w) + \beta\norm{\tilde \w - \w}^2\,.
\label{eq:h-bound}
\end{align}
Hence, by \cref{lem:stab_gen_wor},
\begin{align}\label{eq:wor_stab_main}
\av{\E_\perm\sbr{f(\w_t;\perm_t) -\hat f_{0:t-1}(\w_t)}}
&
=
\av{\frac1{t}\sum_{i=0}^{t-1} \E_\perm\sbr{
        f(\w_t;\perm_t) - f(\w_t^{(\swapit)};\perm_t)}
    }
\notag \\
\explain{\text{Jensen}}
&\leq
\frac1{t}\sum_{i=0}^{t-1} \E_\perm\av{
        {f(\w_t;\perm_t) - f(\w_t^{(\swapit)};\perm_t)}
    }
\notag \\
\explain{\text{\eqref{eq:h-bound}}}
    &\leq
\frac1{t}\sum_{i=0}^{t-1} \E_\perm\sbr{f(\w_t; \perm_t) + \beta \norm{\w_t^{(\swapit)} - \w_t}^2}
\notag \\
&=
\E_\perm f(\w_t, \perm_t)
+ \frac\beta {t} \sum_{i=0}^{t-1} \E_\perm \norm{\w_t^{(\swapit)} - \w_t}^2
.
\end{align}
Next, we bound $\norm{\w_t^{(\swapit)} - \w_t}^2$.
For any $0\leq \tau\leq t-1$, 
we denote $f_\tau \eqq f(\cdot; \perm_\tau)$, and $f_\tau^{(\swapit)}\eqq f(\cdot; \perm(\swapit)_\tau)$. Observe that for any $\tau$ such that $\tau \neq i$, we have $f_\tau = f_\tau^{(\swapit)}$, thus, by the non-expansiveness of gradient steps in the convex and $\beta$-smooth regime when $\eta\leq 2/\beta$ (see Lemma 3.6 in \citealp{hardt2016train}):
\begin{align*}
        \tau \leq i 
        &\implies 
        \norm{\w_{\tau}^{(\swapit)} - \w_{\tau}} = 0,
        \\
        i < \tau
        &\implies
	\norm{\w_{\tau+1}^{(\swapit)} - \w_{\tau+1}}^2
	\leq 
	\norm{\w_{i+1}^{(\swapit)} - \w_{i+1}}^2.
\end{align*}
Further, 
\begin{align*}
	\norm{\w_{i+1}^{(\swapit)} - \w_{i+1}}^2
        &
        =
	\norm{
            \w_{i}^{(\swapit)} - 
            \eta \nabla f_i^{(\swapit)}(\w_{i}^{(\swapit)})
            -\prn{\w_{i}-\eta\nabla f_i(\w_{i})}
        }^2
        \\
        \explain{\w_{i}^{(\swapit)} = \w_{i}}
        &
        =
        \eta^2
	\norm{
             \nabla f_i^{(\swapit)}(\w_{i}^{(\swapit)})
            -
            \nabla f_i(\w_{i})
        }^2
        \\
        \explain{\text{Jensen}}
	&\leq 
	2\eta^2\norm{\nabla f_i^{(\swapit)}(\w_{i}^{(\swapit)})}^2
	+ 2\eta^2\norm{\nabla f_i(\w_{i})}^2
	\\
	&\leq 
	4\beta\eta^2f_i^{(\swapit)}(\w_{i}^{(\swapit)})
	+ 4\beta\eta^2 f_i(\w_{i})\,,
\end{align*}
and by \cref{lem:permutation_equiv} $\E f_i(\w_{i})=\E f_i^{(\swapit)}(\w_{i}^{(\swapit)})$.
Hence,
\begin{align*}
	\E \norm{\w_t^{(\swapit)} - \w_t}^2
        \leq 
        \E \norm{\w_{i+1}^{(\swapit)} - \w_{i+1}}^2
	&\leq 
	8\beta\eta^2\E f_i(\w_{i})\,.
\end{align*}
Now,
\begin{align*}
	\frac\beta {t}\sum_{i=0}^{t-1} 
            \E_\perm \norm{\w_t^{(\swapit)} - \w_t}^2
	&\leq 
	\br{8\beta^2\eta^2}\E\sbr{\frac{1} {t}\sum_{i=0}^{t-1} f_i(\w_i)},
\end{align*}
which, when combined with \cref{eq:wor_stab_main} yields:
\begin{align*}
	\av{\E_\perm\sbr{f(\w_t;\perm_t) 
            -\hat f_{0:t-1}( \w_t)}}
        &\leq 
        \E_\perm f(\w_t; \perm_t)
        +
        \br{8\beta^2\eta^2}\E\sbr{\frac{1} {t}\sum_{i=0}^{t-1} f_i(\w_i)}
        .
\end{align*}
    Finally, by the regret bound given in \cref{lem:regret_bound},
    $\sum_{i=0}^{t-1} f_i(\w_i)
    \leq 
    \frac{\norm{\w_0 - \teacher}^2}{2\eta(2-\eta\beta)}$,
    and therefore,
    \begin{align*}
    	\av{\E_\perm\sbr{f(\w_t;\perm_t) 
                -\hat f_{0:t-1}( \w_t)}}
    	&\leq 
            \E_\perm f(\w_t; \perm_t)
            +
            \frac{4\beta^2\eta\norm{\w_0 - \teacher}^2}{(2-\eta\beta)t}
            \\
            \implies
            \E \hat f_{0:t-1}( \w_t)
            &\leq 
            2\E_\perm f(\w_t; \perm_t)
            +
            \frac{4\beta^2\eta\norm{\w_0 - \teacher}^2}{(2-\eta\beta)t}.
    \end{align*}
    Finally, since $\hat f_{0:t} = \frac{t}{t+1}\hat f_{0:t-1} + \frac{1}{t+1} f_t$,
    we obtain
    \begin{align*}
            \E \hat f_{0:t}( \w_t)
            =
            \frac{t}{t+1}\E \hat f_{0:t-1}( \w_t)
            +
            \frac{1}{t+1}\E f_{t}( \w_t)
            &\leq 
            \frac{2t+1}{t+1}\E_\perm f(\w_t; \perm_t)
            +
            \frac{4\beta^2\eta\norm{\w_0 - \teacher}^2}{(2-\eta\beta)(t+1)},
    \end{align*}
    which completes the proof.
\end{proof}

\newpage
\section{Supplementary Material for the Extension Section (\cref{sec:extensions})}
\label{app:extensions_proofs}

\begin{recall}[\cref{reduc:pocs}]
Consider $T$ arbitrary (nonempty) closed convex sets $\convexset_{1},\dots,\convexset_{T}$,
initial point $\w_0\in\R^{d}$, and ordering $\tau$.
Define $f_{m}(\w) = 
\frac{1}{2}\norm{\w - \proj_{m}(\w)}^{2}, \forall m\in\cnt{T}$.
Then,
\vspace{-0.1em}
\begin{enumerate}[label=(\roman*), itemindent=-0.5cm, labelsep=0.2cm]\itemsep1pt
    \item $f_m$ is convex and $1$-smooth.
    \item The POCS update is equivalent to an SGD step:
    $
\w_t 
=
\proj_{\tau(t)} (\w_{t-1})
=
\w_{t-1} 
-
\nabla_{\w}
f_{\tau(t)} (\w_{t-1})$.
\end{enumerate}
\end{recall}

\medskip

\begin{proof}
First, by Theorem 1.5.5 in \citet{facchinei2003finite}, $f_m$ is continuously differentiable and for every $\w\in \R^d,m\in\cnt{T}$, $\nabla f_m(\w)=\w-\proj_m(\w)$.
Plugging in $\nabla f_{\tau(t)}(\w_{t-1})$ into an appropriate SGD step, we get
$$\w_{t}
=
\w_{t-1} 
-
\nabla_{\w}
f_{\tau(t)} (\w_{t-1})
=\w_{t-1}-\left(\w_{t-1}-\proj_{\tau(t)}(\w_{t-1})\right)
=
\proj_{\tau(t)}(\w_{t-1})\,,$$
and the second part of the lemma follows.
In addition, $\forall \x,\w\in\R^d$, 
we prove convexity by using a projection inequality (also from Theorem 1.5.5 in \citealp{facchinei2003finite}). That is,
\begin{align*}
    &f_m(\x)-f_m(\w)-\langle \nabla f_m(\w),\x-\w\rangle
    \\
    &
    = \frac{1}{2}\|\x-\proj_m(\x)\|^2- \frac{1}{2}\|\w-\proj_m(\w)\|^2- \langle \w-\proj_m(\w),\x-\w\rangle
    \\&
    = 
    \frac{1}{2}\|\x-\proj_m(\x)\|^2- \frac{1}{2}\|\w-\proj_m(\w)\|^2-
    \langle \w-\proj_m(\w),\x-\proj_m(\x)\rangle
    \\&
    \hspace{3.5em}
    +\langle \w-\proj_m(\w),\proj_m(\w)-\proj_m(\x)\rangle
    +\langle \w-\proj_m(\w),\w-\proj_m(\w)\rangle
    \\&
    \geq 
    \frac{1}{2}\|\x-\proj_m(\x)\|^2- \frac{1}{2}\|\w-\proj_m(\w)\|^2-
    \langle \w-\proj_m(\w),\x-\proj_m(\x)\rangle
    +0+\|\w-\proj_m(\w)\|^2
    \\&
    =
    \frac{1}{2}\|\x-\proj_m(\x)-\w+\proj_m(\w)\|^2\geq 0
    \,.
\end{align*}
For the $1$-smoothness, 
\begin{align*}
\norm{\nabla f_m(\x) - \nabla f_m(\w)}
&=
\norm{\x-\proj_{m}(\x) - \prn{\w-\proj_{m}(\w)}}
\\
&
=
\norm{(\I-\proj_{m})(\x) - (\I-\proj_{m})(\w)}
\le
\norm{\x-\w}\,,
\end{align*}
where we used the non-expansiveness of $\I-\proj_{m}$
\citep[Propositions~4.2,~4.8 in][]{bauschke2011convexBook}.
\end{proof}

\newpage

\begin{lemma}\label{lem:proj_gradineq}
    Let $\cK\subseteq\R^d$ be a nonempty closed and convex set, and
     $f(\w) = \frac12 \norm{\w - \proj_\cK(\w)}^2$.
    Then, we have for any $\vz\in \R^d$ and $\gamma>0$
    \begin{align*}
        (2-\gamma)f(\w) - \frac1\gamma f(\vz)
    &\leq \nabla f(\w)\T (\w - \vz)\,.
    \end{align*}
    In addition, for any $\vu\in\cK$ we have
    \begin{align*}
    2 f(\w) 
    &\leq \nabla f(\w)\T (\w - \vu)\,.
    \end{align*}
\end{lemma}
\medskip
\begin{proof}
    We already established that $\nabla f(\w) = \w - \proj_\cK(\w)$. 
    Combining this with simple algebra, we obtain, 
    \begin{align*}
        \abr{\nabla f(\w), \w - \vz}
        &=
        \abr{\w - \proj_\cK(\w), \w - \vz}
        \\
        &
        =
        \abr{\w - \proj_\cK(\w), \w - \proj_\cK(\w)}
        +
        \abr{\w - \proj_\cK(\w), \proj_\cK(\w) - \vz}
        \\
        &
        =
        2f(\w)
        +
        \abr{\w - \proj_\cK(\w), \proj_\cK(\w) - \vz}
        \\
        &=
        2f(\w)
        +
        \abr{\w - \proj_\cK(\w), \proj_\cK(\w) - \proj_\cK(\vz)}
        -
        \abr{\w - \proj_\cK(\w), \vz - \proj_\cK(\vz)}
        \,.
    \end{align*}
By Theorem 1.5.5 (b) in \citet{facchinei2003finite},
we have that 
$$\abr{\w - \proj_\cK(\w), \proj_\cK(\w) - \proj_\cK(\vz)} \ge 0
\,,$$ and finally we get,
    \begin{align*}
        \abr{\nabla f(\w), \w - \vz}
        &\geq
        2f(\w)
        -
        \abr{\w - \proj_\cK(\w), \vz - \proj_\cK(\vz)}\,.
    \end{align*}
    Plugging in $\vz=\vu$, the second term vanishes
    (since $\vu-\proj_\cK(\vu)=\0$)
    and the second claim follows.
    
    \bigskip
    For the first claim, note that by Young's inequality:
    \begin{align*}
    \abr{\nabla f(\w), \w - \vz}
    &=
        2f(\w)
        -
        \abr{\w - \proj_\cK(\w), \vz - \proj_\cK(\vz)}
        \\
        &
        \geq
        2f(\w)
        -
        \frac\gamma2 \norm{\w - \proj_\cK(\w)}^2 
        -\frac1{2\gamma} \norm{\vz - \proj_\cK(\vz)}^2
        \\
        &
        =
        2f(\w)
        -
        \gamma f(\w)
        -\frac1\gamma f(\vz)\,.
    \end{align*}
\end{proof}

\newpage

\begin{recall}[\cref{thm:pocs-rate}]
Consider the same conditions of \cref{reduc:pocs} and assume a nonempty set intersection
$\convexintersection
=
\bigcap_{m=1}^{T} \convexset_{m}
\neq 
\varnothing
$.
Then, under a random ordering with or without replacement, 
the expected ``residual'' of \cref{scheme:pocs} after
$\forall k \ge 1$ iterations (without replacement: $k\in\cnt{T}$) is bounded as,
\begin{align*}
\mathbb{E}_{\tau}
\Big[
\frac{1}{2T}
\tsum_{m=1}^{T} 
\norm{\w_{k} - \proj_{m}(\w_{k})}^{2}
\Big]
=
\mathbb{E}_{\tau}
\Big[
\frac{1}{2T}
\tsum_{m=1}^{T} 
\mathrm{dist}^2 (\w_{k}, \convexset_{m})
\Big]
\le 
\frac{7}{\sqrt[4]{k}}
\min_{\w\in \convexintersection}
\norm{\w_{0} - \w}^2\,.
\end{align*}
\end{recall}
\begin{proof}
    The proof largely follows the same steps of \cref{thm:cl_by_sgd_main,thm:cl_by_sgd_wor}.
        Let $\tau$ be any random ordering, $\w_0\in\R^d$ an initialization, and $\w_1, \ldots, \w_k$ be the corresponding iterates produced by \cref{scheme:pocs}.
    By \cref{reduc:pocs}, these are exactly the (stochastic) gradient descent iterates produced when initializing at $\w_0$ and using a step size of $\eta=1$, on the $1$-smooth loss sequence $f_{\tau(1)}, \ldots, f_{\tau(k)}$  defined by:
    \begin{align*}
        f_m(\w) \eqq \frac12\norm{\w - \proj_m(\w))}^2.
    \end{align*}
    Proceeding, we denote the objective function:
    \begin{align*}
        \f(\w) \eqq \E_{m\sim \Unif([T])} f_m(\w)
        =
        \frac{1}{2T}
        \sum_{m=1}^{T} 
        \norm{\w - \proj_{m}(\w)}^{2}\,.
    \end{align*}
    Now, for a \textbf{with-replacement} ordering $\tau$,
    invoke \cref{thm:sgd_last_iterate_main}, except we use \cref{lem:proj_gradineq} in the proof instead of \cref{lem:sqloss_gradineq}, to obtain:
    \begin{align*}
        \E_{\tau} \f(\w_k) 
        \leq 
        \frac{e}{2\sqrt[4]{k}}
        \min_{\w\in \convexintersection}
        \norm{\w_{0} - \w}^2
        \,,
        \tag{$\tau$ with-replacement}
    \end{align*}
    which completes the proof for the with-replacement case.

    \medskip

    For a \textbf{without-replacement} ordering $\tau$,
    invoke \cref{thm:wor_sgd_last_iterate_main} (with $\eta=1/\beta$), except again we use \cref{lem:proj_gradineq} in the proof instead of \cref{lem:sqloss_gradineq}, to obtain:
    \begin{align*}
        \E_\tau \hat f_{0:k-1}(\w_k) 
        \triangleq
        \E_\tau \Big[\frac1{k}\sum_{t=0}^{k-1} f(\w_{k})
        \Big]
        \leq 
        \frac{7}{\sqrt[4]{k}}
        \min_{\w\in \convexintersection}
        \norm{\w_{0} - \w}^2
        \,.
        \tag{$\tau$ without-replacement}
    \end{align*}
    Similarly, by \cref{lem:wor_last_before_stab},
    \begin{align*}
        \mathbb{E}_{\tau} f_{\tau(k+1)}(\w_k)
        \triangleq
        \mathbb{E}_{\tau}
        \tfrac{1}{2}
        \left\Vert \w_{k}-\proj_{\tau\left(k+1\right)}\left(\w_{k}\right)\right\Vert ^{2}
        \leq 
        \frac{e}{2\sqrt[4]{k}}
        \min_{\w\in \convexintersection}
        \norm{\w_{0} - \w}^2\,.
        \tag{$\tau$ without-replacement}
    \end{align*}
    Combining the last two displays with \cref{prop:loss_to_forgetting}, we now obtain:
    \begin{align*}
        \E_\tau \f(\w_k) 
        &\eqq
        \E_\tau \Big[
        \frac{1}{2T}
        \sum_{m=1}^{T} 
        \norm{\w_k - \proj_{m}(\w_k)}^{2}
        \Big]
        \tag{$\tau$ without-replacement}
        \\
        &
        =
        \frac{k}{T}
        \E_\tau \hat f_{0:k-1}(\w_k)
        +
        \frac{T-k}{2T}
        \mathbb{E}_{\tau}\left\Vert \w_{k}-\proj_{\tau\left(k+1\right)}\left(\w_{k}\right)\right\Vert ^{2}
        \\
        &
        \leq 
        \prn{
        \frac{7k}{T}
        +
        \frac{\frac{e}{2}(T-k)}{T}
        }
        \frac{1}{\sqrt[4]{k}}
        \min_{\w\in \convexintersection}
        \norm{\w_{0} - \w}^2
        \leq 
        \frac{7}{\sqrt[4]{k}}
        \min_{\w\in \convexintersection}
        \norm{\w_{0} - \w}^2
        \,,
    \end{align*}
    which proves the without-replacement case and thus completes the proof.
\end{proof}

\newpage

\begin{recall}[\cref{cor:pocs_forgetting}]
Under a random ordering, with or without replacement, over $T$ jointly separable tasks, the expected forgetting of the weakly-regularized \cref{scheme:regularized_cl} (at $\lambda\to 0$) 
after $k \ge 1$ iterations (without replacement: $k\in\cnt{T}$) is bounded as
\begin{align*}
\mathbb{E}_{\tau}
\big[
F_{\tau}({k})
\big]
\le 
\frac{7\norm{\teacher}^2
R^2}{\sqrt[4]{k}}
\,,
\quad
\text{where }
\teacher \triangleq 
{\min}_{\w\in \convexset_1\cap\dots \cap\convexset_T} \norm{\w_{0} - \w}^2
\,.
\end{align*}
\end{recall}

\begin{proof}
We adopt the same notation as used above:
\begin{align*}
        f_m(\w) &\eqq \frac12\norm{\w - \proj_m(\w))}^2
        \\
        \f(\w) &\eqq \E_{m\sim \Unif([T])} f_m(\w)
        =
        \frac{1}{2T}
        \sum_{m=1}^{T} 
        \norm{\w - \proj_{m}(\w)}^{2}\,.
    \end{align*}
    For $\tau$ sampled \textbf{with replacement}, 
    by \cref{lem:wr_stability} (given below) and the with-replacement result (inside the proof) of \cref{thm:pocs-rate}, we have 
    \begin{align*}
        \E_\tau [F_\tau(k)] = \E \hat f_{0:k-1}(\w_k)
        &
        \leq 
        2 \E \f(\w_k)
        +
        \frac{4\norm{\w_0 - \teacher}^2}{k}
        \\
        &
        \leq 
        \prn{
        \frac{e}{\sqrt[4]{k}}
        + \frac{4}{k}
        }\norm{\w_0 - \teacher}^2
        \leq 
        \frac{7\norm{\w_0 - \teacher}^2}{\sqrt[4]{k}}.
    \end{align*}
    For $\tau$ sampled \textbf{without replacement}, as argued in \cref{thm:pocs-rate}, by \cref{lem:wor_last_before_stab}:
    \begin{align*}
        \mathbb{E}_{\tau} f_{\tau(k+1)}(\w_k)
        \leq 
        \frac{\frac{e}{2}\norm{\w_0 - \teacher}^2}{\sqrt[4]{k}}
        \,,
    \end{align*}
    and thus by \cref{lem:wor_generalization},
    \begin{align*}
        \E_\tau [F_\tau(k)] = \E \hat f_{0:k-1}(\w_k)
        \leq 
        \prn{
        \frac{e}{\sqrt[4]{k}}
        + \frac{4}{k}
        }\norm{\w_0 - \teacher}^2
        \leq 
        \frac{7\norm{\w_0 - \teacher}^2}{\sqrt[4]{k}}
        \,.
    \end{align*}
    which completes the proof.
\end{proof}

\bigskip

\begin{lemma}\label{lem:wr_stability}
Consider with-replacement SGD \cref{def:sgd_withreplacement} with step size $\eta\leq2/\beta$, and define, for every $0\leq T$, $\hat f_{0:T}(\w) \eqq \frac1{T+1}\sum_{t=0}^T f(\w; i_t)$.
For all $1\leq T$, the following holds:
\begin{align*}
    	\E \hat f_{0:T-1}(\w_T)
    	&\leq 
            2 \E \f(\w_T)
            +
            \frac{4\beta^2\eta\norm{\w_0 - \teacher}^2}{T}
            \,.
    \end{align*}
\end{lemma}
\begin{proof}
Our proof here mostly follows the proof of \cref{lem:wor_generalization}.
Recall that from \eqref{eq:h-bound}, any $\beta$-smooth realizable function $h:\R^d\to\R_{\ge 0}$ holds that
$\av{h(\tilde \w) - h(\w)}
\leq h(\w) + \beta\norm{\tilde \w - \w}^2$.
Denote $f_t \eqq f(\cdot; i_t)$ for all $t\in\cbr{0, \smalldots, T}$.
Now, by the standard stability $\iff$ generalization argument \citep{shalev2010learnability,hardt2016train}, and denoting by $\w_\tau^{(i)}$ the SGD iterate after $\tau$ steps on the training set where the $i$\nth example was resampled as $j_i$:
\begin{align*}
\av{\E \sbr{\f(\w_T) - \hat f_{0:T-1}(\w_T)}}
&=
\Big|{\frac1{T}\sum_{i=0}^{T-1} \E_{j_i \sim \cD}\sbr{
        f(\w_T; j_i) - f(\w_T^{(i)};j_i)}
    }\Big|
\\
\explain{\text{Jensen; \eqref{eq:h-bound}}}
    &\leq
\frac1{T}\sum_{i=0}^{T-1} \E\sbr{f(\w_T; j_i) + \beta \norm{\w_T^{(i)} - \w_T}^2}
\\
&
=
\E \f(\w_T)
+ \frac\beta T \sum_{i=0}^{T-1} 
    \E \norm{\w_T^{(i)} - \w_T}^2
.
\end{align*}
Next, we bound $\norm{\w_T^{(i)} - \w_T}^2$. By the non-expansiveness of gradient steps in the convex and $\beta$-smooth regime when $\eta\leq 2/\beta$ \citep[see Lemma 3.6 in][]{hardt2016train}:
\begin{align*}
        \tau \leq i 
        &\implies 
        \norm{\w_{\tau}^{(i)} - \w_{\tau}} = 0,
        \\
        i < \tau
        &\implies
	\norm{\w_{\tau+1}^{(i)} - \w_{\tau+1}}^2
	\leq 
	\norm{\w_{i+1}^{(i)} - \w_{i+1}}^2.
\end{align*}
Further,
\begin{align*}
	\norm{\w_{i+1}^{(i)} - \w_{i+1}}^2
        &
        =\norm{\w_{i}^{(i)}-
        \eta \nabla f_{j_i}(\w_{i}^{(i)})
        - 
        \prn{\w_{i}
        -\eta\nabla f_i(\w_{i})
        }
        }^2
        \\
        \explain{\w_{i}^{(i)} = \w_{i}}
        &
        =
        \eta^2
        \norm{
        \nabla f_{j_i}(\w_{i}^{(i)})
        - \nabla f_i(\w_{i})
        }^2
        \\
        \explain{\text{Jensen}}
	&\leq 
        2\eta^2\bignorm
        {\nabla f_{j_i}(\w_{i}^{(i)})}^2
	+ 
        2\eta^2\bignorm
        {\nabla f_i(\w_{i})}^2
	\\
        \explain{\text{smoothness,}\\ \text{non-negativity}}
	&\leq 
	4\beta\eta^2f_{j_i}(\w_{i}^{(i)})
	+ 4\beta\eta^2 f_i(\w_{i})\,.
\end{align*}
Therefore,
\begin{align*}
	\E \norm{\w_T^{(i)} - \w_T}^2
        \leq 
        \E \norm{\w_{i+1}^{(i)} - \w_{i+1}}^2
	&\leq 
        4\beta\eta^2\E f_{j_i}(\w_{i}^{(i)})
	+ 4\beta\eta^2 \E f_i(\w_{i})
        =
	8\beta\eta^2\E f_i(\w_{i})\,.
\end{align*}
Now,
\begin{align*}
\frac\beta {T}\sum_{i=0}^{T-1} 
        \E \norm{\w_T^{(i)} - \w_T}^2
&\leq 
{12\beta^2\eta^2}\,
    \E\sbr{\frac{1} {T}\sum_{i=0}^{T-1} f_i(\w_i)}.
\end{align*}
Summarizing, we have shown that:
\begin{align*}
	\av{\E\sbr{\f(\w_T) 
            -\hat f_{0:T-1}( \w_T)}}
	&\leq
	\E \f(\w_T)
	+ \frac\beta {T}\sum_{i=0}^{T-1} \E\norm{\w_T^{(i)} - \w_T}^2
	\\
        &\leq 
        \E \f(\w_T)
        +
        {8\beta^2\eta^2}\,\E\sbr{\frac{1} {T}\sum_{i=0}^{T-1} f_i(\w_i)}
        .
\end{align*}
    Finally, by the regret bound given in \cref{lem:regret_bound},
    \ie
        $\sum_{i=0}^{T-1} f_i(\w_i)
        \leq 
        \frac{\norm{\w_0 - \teacher}^2}{2\eta(2-\eta\beta)}$,
    we have
    \begin{align*}
    	\av{\E\sbr{\f(\w_T) 
            -\hat f_{0:T-1}( \w_T)}}
    	&\leq 
            \E \f(\w_T)
            +
            \frac{4\beta^2\eta\norm{\w_0 - \teacher}^2}{(2-\eta\beta)T}
            \,.
    \end{align*}
    and the result follows.
\end{proof}
\newpage

\section{Supplementary Material for the Discussion Section (\cref{sec:discussion})}
\label{app:discussion_proofs}

\begin{claim}[Average-Norm Universal Rate for With-Replacement Random Ordering]
\label{prop:mean_norm}
Under a random ordering with replacement 
over $T$ jointly realizable tasks, 
the expected loss and forgetting of {Schemes~\ref{proc:regression_to_convergence},~\ref{scheme:kaczmarz}}
    after $k\geq 2$ iterations are bounded as,
\begin{align*}
\mathbb{E}_{\tau}\!\left[\Loss\left(\w_{k}\right)\right]
&
\le
\frac{2\norm{\teacher}^2 \bar R}{ \sqrt[4]{k}}
\,,\quad\quad
\mathbb{E}_{\tau}\!
\left[
F_{\tau}(k)
\right]
\le 
\frac{5\norm{\teacher}^2 \bar R}{ \sqrt[4]{k-1}}
\,,
\end{align*}
where $\bar R=\sum_{m=1}^M \norm{\X_m}^2/T$.
\end{claim}
\begin{proof-sketch}
Taking the non-worst case bound from 
\cref{lem:cl_gd_equiv}, we have $\mathcal L_m(w) \leq \alpha_m f_m(w)$ for $\alpha_m = \Vert \X_m \Vert^2$. 
Then in the proof of \cref{thm:cl_by_sgd_main}, $\mathcal L(w) \leq \frac{A}{2T} \sum \frac{\alpha_m}{A} f_m(w)$, where $A=\sum \alpha_m$, and we may apply \cref{thm:sgd_last_iterate_main} (which supports arbitrary distributions $\mathcal D$, see Setup 1) with the distribution given by $\Pr_{\mathcal D}(i) = \alpha_i/A$. Finally, we have $\bar R = A / T$.
\end{proof-sketch}

\end{document}